\documentclass{article} 
\usepackage{iclr2024_conference,times}

\usepackage{amsmath,amsfonts,bm}

\def\eqref#1{equation~\ref{#1}}

\def\1{\bm{1}}

\DeclareMathAlphabet{\mathsfit}{\encodingdefault}{\sfdefault}{m}{sl}
\SetMathAlphabet{\mathsfit}{bold}{\encodingdefault}{\sfdefault}{bx}{n}

\newcommand{\E}{\mathbb{E}}

\newcommand{\R}{\mathbb{R}}

\usepackage{hyperref}
\usepackage{url}
\usepackage{enumitem}
\setlist{nolistsep}

\newcommand{\smartparagraph}[1]{\vspace{2pt} \noindent {\bf #1}}

 \usepackage{graphicx} 
\usepackage{layout}

\usepackage{amsfonts}
\usepackage{amsmath}
\usepackage{amsthm}
\usepackage{amssymb} 

\usepackage{mdframed} 
\usepackage{thmtools}

\definecolor{shadecolor}{gray}{0.95}
\declaretheoremstyle[
headfont=\normalfont\bfseries,
notefont=\mdseries, notebraces={(}{)},
bodyfont=\normalfont,
postheadspace=0.5em,
spaceabove=1pt,
mdframed={
  skipabove=8pt,
  skipbelow=8pt,
  hidealllines=true,
  backgroundcolor={shadecolor},
  innerleftmargin=4pt,
  innerrightmargin=4pt}
]{shaded}

\usepackage{xcolor}
\usepackage{color}

\usepackage{verbatim}

\newcommand{\mbE}{\mathbb{E}} 

\newcommand{\cA}{{\cal A}}
\newcommand{\cB}{{\cal B}}
\newcommand{\cC}{{\cal C}}
\newcommand{\cD}{{\cal D}}

\newcommand{\cG}{{\cal G}}

\newcommand{\xs}{x_\star}
\newcommand{\fs}{F_\star}

\newcommand{\eqdef}{:=} 

\newcommand{\dotprod}[1]{\left< #1\right>} 
\newcommand{\norm}[1]{\lVert#1\rVert}      

\newtheorem{assumption}{Assumption}
\newtheorem{lemma}{Lemma}

\newtheorem{theorem}{Theorem}

\newtheorem{corollary}{Corollary}

\theoremstyle{definition}

\theoremstyle{remark}
\newtheorem{remark}{Remark} 

 \newcommand{\myNum}[1]{(\emph{#1})}
 \usepackage{tcolorbox}
\newtcolorbox{myblock}[2][]{colback=lightgray,colframe=gray,#1,title={#2}}

\title{Demystifying the Myths and Legends of Nonconvex Convergence of SGD}

\author{Aritra Dutta\\
Department of Mathematics\\
University of Central Florida\\
Orlando, FL 32816, USA \\
\texttt{aritra.dutta@ucf.edu} \\
\And
El Houcine Bergou \& Soumia Boucherouite\\
College of Computing \\
Mohammed VI Polytechnic University  \\
Benguerir, Morocco \\
\texttt{\{elhoucine.bergou,soumia.boucherouite\}@um6p.ma} \\
\And
Nicklas Werge \& Melih Kandemir \\
Department of Mathematics and Computer Science \\
University of Southern Denmark \\
Odense, Denmark \\
\texttt{werge@sdu.dk, kandemir@imada.sdu.dk} \\
\And
Xin Li\\
Department of Mathematics\\
University of Central Florida\\
Orlando, FL 32816, USA \\
\texttt{xin.li@ucf.edu} \\
}

\iclrfinalcopy 
\begin{document}

\maketitle

\begin{abstract}
Stochastic gradient descent (SGD) and its variants are the main workhorses for solving large-scale optimization problems with nonconvex objective functions.~Although the convergence of SGDs in the (strongly) convex case is well-understood, their convergence for nonconvex functions stands on weak mathematical foundations. Most existing studies on the nonconvex convergence of SGD show the complexity results based on either the minimum of the expected gradient norm or the functional sub-optimality gap (for functions with extra structural property) by searching the entire range of iterates. Hence the last iterations of SGDs do not necessarily maintain the same complexity guarantee. This paper shows that {\em an $\epsilon$-stationary point exists in the final iterates of SGDs,} given a large enough total iteration budget, $T$, not just anywhere in the entire range of iterates --- a much stronger result than the existing one. Additionally, our analyses allow us to measure the \emph{density of the $\epsilon$-stationary points} in the final iterates of SGD, and we recover the classical ${O(\frac{1}{\sqrt{T}})}$ asymptotic rate under various existing assumptions on the objective function and the bounds on the stochastic gradient. As a result of our analyses, we addressed certain myths and legends related to the nonconvex convergence of SGD and posed some thought-provoking questions that could set new directions for research. 
\end{abstract}

\tableofcontents 

\section{Introduction}
We consider the {\em empirical risk minimization}~(ERM) problem:
\begin{equation}\label{eq:opt}
\min_{x\in\R^d} \left[F(x) \eqdef \frac{1}{n}\sum_{i=1}^nf_i(x)\right],
\end{equation}
where $f_i(x)\eqdef\mbE_{z_i \sim \cD_i}{l(x;z_i)}$ denotes the loss function evaluated on input, $z_i,$ sampled from its distribution, $\cD_i.$ Additionally, let $F$ be nonconvex, lower bounded, with Lipschitz continuous gradient; see \S \ref{sec:assumption}. ERM problems appear frequently in statistical estimation and machine learning, where the parameter, $x$, is estimated by the SGD updates \citep{doi:10.1137/16M1080173,grace}.~For a sequence of iterates, $\{x_t\}_{t\ge 0}$ and a stepsize parameter, $\gamma_t>0$, SGD updates are of the form: 
\begin{eqnarray}\label{iter:sgd}
x_{t+1} = x_t - \gamma_t g_t, 
\end{eqnarray}
where $g_t$ is an unbiased estimator of $\nabla F(x_t)$, the gradient  of $F$ at $x_t$; that is, $\textstyle{\mathbb{E}(g_t|x_t)=\nabla F(x_t).}$ 
This approach, as given in (\ref{iter:sgd}), can be implemented by selecting an index $i(t)$ independently and uniformly 
from the set $[n],$ and processes $g_t=\nabla f_{i(t)}(x_t)$; the same index can be selected again. 


The convergence of SGD 
for the strongly convex functions is well understood  \citep{shalev2009stochastic,pmlr-v97-qian19b, shamir13}, but its convergence of the {\em last} iterates for nonconvex functions remains an open problem.
~For a nonconvex function, $F$, the existing convergence analyses of SGD show, as ${T\to\infty}$, either \myNum{i} the minimum of the norm of the gradient function, ${\min_{t\in[T]}\mathbb{E}\|\nabla F(x_t)\|\to 0}$ \citep{ghadimi2013stochastic,khaled2020better,stich2020error}\footnote{Some works show, as ${T\to\infty}$, the average of the expected gradient norm, ${\frac{1}{T}\sum_t\mathbb{E}\|\nabla F(x_t)\|\to 0}$ for fixed stepsize, or the weighted average of the expected gradient norm, ${\frac{1}{\sum_t \gamma_t}\sum_t\gamma_t\mathbb{E}\|\nabla F(x_t)\|\to 0}$ for variable stepsize \citep{doi:10.1137/16M1080173}.}  or \myNum{ii} the minimum sub-optimality gap, ${{\min_{t\in[T]}\left(\mathbb{E}(F(x_t))-F_{\star}\right)\to 0}}$ \citep{gower21a, Lei_TNNLS}.
Notably, the first-class uses the classical $L$-smoothness, and the size of the gradient function, ${\mathbb{E}\|\nabla F(x_t)\|}$ to measure the convergence. Whereas the second class considers $F$ to have extra structural property, such as Polyak-\L ojasiewicz (PL) condition \citep{gower21a, Lei_TNNLS} and the minimum sub-optimality gap is the measure of convergence. Nevertheless, in both cases the notion of $\epsilon$-stationary points\footnote{A stationary point is either a local minimum, a local maximum, or a saddle point. In nonconvex convergence of SGD, $x$ is an $\epsilon$-stationary point if $\mathbb{E}\|\nabla F(x)\|\le \epsilon$ or $(\mathbb{E}(F(x))-F_{\star})\le \epsilon$.} is weak as they only consider the minimum of the quantity $\mathbb{E}\|\nabla F(x_t)\|$ or $\left(\mathbb{E}(F(x_t))-F_{\star}\right)$ approaching to 0 (as ${T\to\infty}$) by searching over the entire range of iterates, $[T]$. Alongside, by adding one more random sampling step at the end, ${x_{\tau}\sim\{x_t\}_{t\in[T]}}$, some works show ${\mathbb{E}\|\nabla F _\tau(x_\tau)\|}\to0$ instead \citep{ghadimi2013stochastic, stich2020error, wang2019stochastic}. 

In practice, we run SGD for $T$ iterations (in the order of millions for DNN training) and return the last iterate \citep{shalev2011pegasos}. Therefore, we may ask: {\em How practical is the notion of an $\epsilon$-stationary point?} E.g., training ResNet-50 \citep{he2016deep} on ImageNet dataset \citep{imagenet} requires roughly ${600,000}$ iterations. The present nonconvex convergence analysis of SGD tells only us that at one of those ${600,000}$ iterations, $\mathbb{E}\|\nabla F(x_t)\|\approx 0$. Indeed, this is not a sufficiently informative outcome. That is, the existing results treat all the iterations equally, do not motivate why we need to keep producing more iterations, but only reveal that as long we are generating more iterations, one of them will be $\epsilon$-stationary point. However, numerical experiments suggest that running SGD for more iterations will produce $\epsilon$-stationary points progressively more frequently, and the final iterates of SGD will surely contain more $\epsilon$-stationary points. Motivated by the remarkable fact that {\em the research community chooses high iteration counts to ensure convergence}, we investigate the theoretical foundations that justify the success of this common practice by asking the following questions: \emph{Can we guarantee the existence of $\epsilon$-stationary point for SGD for the nonconvex case in the tail of the iterates, given a large enough iteration budget, $T$?} But guaranteeing the final iterates of SGD contain one of the $\epsilon$-stationary points alone does not conclude the task: we also would like to quantify the denseness of these  $\epsilon$-stationary points among the last iterates.~In all cases, in $T$ iterations, SGD achieves an optimal ${O(\frac{1}{\sqrt{T}})}$ asymptotic convergence rate for nonconvex, $L$-smooth functions \citep{carmon2020lower}. A more refined analysis is required to capture this asymptotic rate. 





We answer these questions and make the following contributions: 
 \begin{itemize}
     \item For any stepsize (constant or decreasing), we show the existence of $\epsilon$-stationary points in the final iterates of SGD for nonconvex functions; see Theorem \ref{thm:general_analysis_sgd:k} in \S\ref{sec:concentartion}.~We show that for every fixed $\eta\in(0,1],$ there exists a $T$, large enough, such that there exists $\epsilon$-stationary point in the final $\eta T$ iterations. We can recover the classic ${O(\frac{1}{\sqrt{T}})}$ convergence rate of nonconvex SGD under no additional assumptions.

     \item An interesting consequence of our analyses is that we can measure the \emph{concentration of the $\epsilon$-stationary points} in the final iterates of SGD---A first standalone result; see \S\ref{sec:concentartion}, Theorem \ref{theorem:density} and \ref{theorem:density-dss}.~That is, we show that the concentration of the $\epsilon$-stationary points over the tail portion for the SGD iterates, $x_t$ is almost 1 for large $T$, where ${t\in [(1-\eta)T,T]}$ and ${\eta\in(0,1]}$.~In optimization research community it is commonly agreed that without abundantly modifying the SGD algorithm, it is impossible to guarantee the convergence of the {\em last iterate} of SGD for nonconvex functions with the classic asymptotic rate \cite[p.17]{dieuleveut2023stochastic}.~Therefore, on a higher level, our result clarifies common intuitions in the community: \myNum{i} why we need to keep on producing more iterations while running SGD; and \myNum{ii} why the practitioners usually pick the {\em last iterate of SGD} instead of randomly picking {\em one of the many iterates} as suggested by the present nonconvex theory. We acknowledge that these points are intuitive but taken for granted in the optimization research community without a theory.~Additionally, we can extend our techniques to the convergence of random-reshuffling SGD (RR-SGD) and SGD for nonconvex and nonsmooth objectives; see~\S \ref{Appendix:rrsgd_conv} and \S\ref{sec:sgd_conv_nonsmooth}. 
     
\item We informally addressed certain myths and legends that shape the practical success of SGD in the nonconvex world. We provide simple and interpretable reasons why some of these directions do not contribute to the success and posed some thought-provoking questions which could set new directions for research. We also support our theoretical results by performing numerical experiments on nonconvex functions, both smooth (logistic regression with nonconvex penalty) and nonsmooth (feed-forward neural network with ReLU activation); see \S \ref{sec:numerical results}. For publicly available code see \S\ref{App:code}.
  
 \end{itemize}

\section{Brief literature review}
The convergence of SGD for convex and strongly convex functions is well understood; see \S\ref{appendix:literature_review} for a few related work. 
We start with nonconvex convergence of SGD. 

\begin{figure*}
    \centering
    \includegraphics[width=1\textwidth]{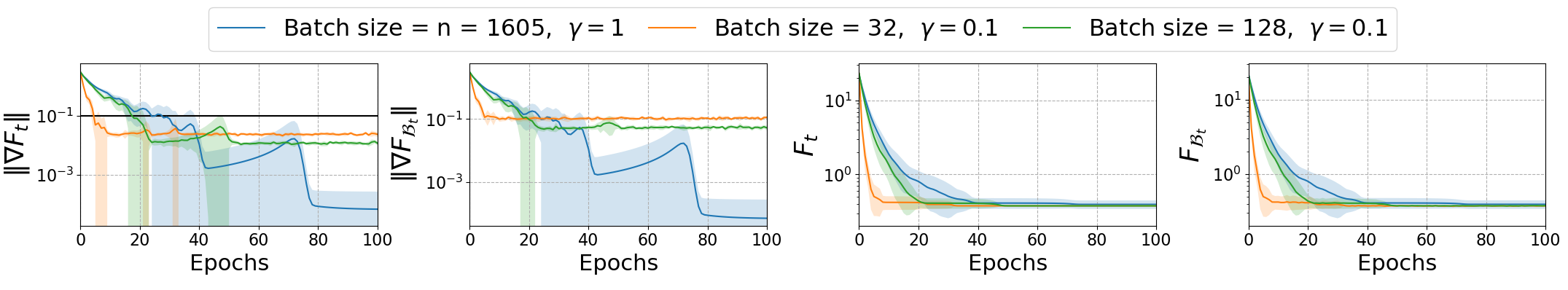}
    \caption{\small{Average of 10 runs of \texttt{SGD} on logistic regression with nonconvex regularization. Batch size, $n=1,605$, represents full batch. In the first column, the horizontal lines correspond to the precision, ${\epsilon=10^{-1}}$, and conform to our theoretical result in Theorem \ref{theorem:main-es}---If the total number of iterations is large enough then almost all the iterates in the tail are $\epsilon$-stationary points.}}
    \label{fig:log_logistic}
\end{figure*}


\textbf{Nonconvex convergence of SGD} was first proposed by \cite{ghadimi2013stochastic} for nonlinear, nonconvex stochastic programming. 
Inspired by \cite{nesterov2003introductory,gratton2008recursive}, \cite{ghadimi2013stochastic} showed that SGD achieves ${\min_{t\in[T]}\mathbb{E}\|\nabla F(x_t)\|^2\le \epsilon}$ after running for at most ${O(\epsilon^{-2})}$ steps---same complexity as the gradient descent for solving \eqref{eq:opt}.~\cite{ghadimi2016mini} extended their results to a class of constrained stochastic composite optimization problems with loss function as the sum of a differentiable (possibly nonconvex) function and a non-differentiable, convex function. Recently, \cite{vaswani2019fast} proposed a {\em strong growth condition} of the stochastic gradient, 
and showed that under SGC with a constant $\rho$, SGD with a constant step-size can attain the optimal rate, ${O(\epsilon^{-1})}$ for nonconvex functions; see Theorem 3 which is an improvement over \cite{ghadimi2013stochastic}. 
\cite{stich2020error} proposed the {\em (${M, \sigma^2}$) noise bound} for stochastic gradients and proposed a convergence analysis for (compressed and/or) error-compensated SGD; see \citep{stich2018sparsified,beta_nips}. For nonconvex functions, \cite{stich2020error} showed ${\mathbb{E}\|\nabla F (x_\tau)\|\to0},$ where ${x_{\tau}\sim\{x_t\}_{t\in[T]}}$. At about the same time, \cite{khaled2020better} proposed a new assumption, {\em expected smoothness}, see Assumption \ref{ass:ABC}, for modelling the second moment of the stochastic gradient and achieved the optimal ${O(\epsilon^{-2})}$ rate for SGD in finding stationary points
for nonconvex, $L$-smooth functions. 
Among others, \cite{Lei_TNNLS} used Holder’s continuity on gradients and showed the nonconvex convergence of SGD. Additionally, they showed the loss, $F$ converges to an {\em almost surely bounded random variable}. By using mini-batches to control the loss of iterates to non-attracted regions, \cite{fehrman2020convergence} proved the convergence of SGD to a minimum for {\em not necessarily} locally convex nor contracting objective functions. Additionally, for convergence of proximal stochastic gradient algorithms (with or without variance reduction) for nonconvex, nonsmooth finite-sum problems, see \cite{j2016proximal, li2018simple}; for non-convex problems with a non-smooth, non-convex regularizer, see \cite{xu2019non}. 


\textbf{Adaptive gradient methods} such as ADAM \citep{kingma2015adam}, AMSGrad \citep{Reddi2018OnTC}, AdaGrad \citep{duchi2011adaptive} are extensively used for DNN training. Although the nonconvex convergence of these algorithms are more involved than SGD, they focus on the same quantities as SGD to show convergence. 
See nonconvex convergence of ADAM and AdaGrad by \cite{defossez2020simple}, nonconvex convergence for AdaGrad by \cite{ward2019adagrad}, Theorem 2.1; also, see Theorem 3 in \citep{zhou2018convergence}, and Theorem 2 in \citep{yang2016unified} for a unified analysis of stochastic momentum methods for nonconvex functions. Recently, \cite{jin2022convergence} proved almost sure asymptotic convergence of momentum SGD and AdaGrad. 

\textbf{Compressed and distributed SGD} is widely studied to remedy the network bottleneck in bandwidth limited training of large DNN models, such as federated learning \citep{FL:Jakub,kairouz2019advances}. The convergence analyses of compressed and distributed SGD for nonconvex loss \citep{layer-wise,grace,alistarh2017qsgd,beta_nips,stich2020error} follow the same structure of the existing nonconvex convergence of SGD. 

\textbf{Structured nonconvex convergence analysis of SGD and similar methods.} \cite{gower21a} used extra structural assumptions on the nonconvex functions and showed SGD converges to a global minimum. 
\cite{gorbunov2021extragradient} showed when $F$ is cocoercive (monotonic and $L$-Lipschitz gradient), 
the last-iterate for extra gradient \citep{korpelevich1976extragradient} and optimistic gradient method \citep{popov1980modification} converge at ${O(\frac{1}{T})}$ rate.

\section{Assumptions}\label{sec:assumption}
\begin{assumption}\label{ass:minimum}
\textbf{(Global minimum)} 
There exists $\xs$ such that $\fs:=F(\xs)\leq F(x)$ for all $x\in\R^d$.  
\end{assumption}
\begin{assumption}\label{ass:smoothness}
\textbf{(Smoothness)} For every $i \in [n]$, the function $f_i: \R^d\to \R$ is $L$-smooth, i.e. $f_i(y)\leq f_i(x)+\dotprod{\nabla f_i(x), y-x}+\frac{L}{2}\norm{y-x}^2$ for all $x,y\in\R^d$. 
\end{assumption}
\begin{remark}\label{remark:L_smooth}
The above implies that $F$ is $L$-smooth.
\end{remark}
\textbf{Bound on stochastic gradient.} There are different assumptions to bound the stochastic gradient. One may follow the model of \cite{stich2020error}:~Let the stochastic gradient in (\ref{iter:sgd}), $g_{t}$, at iteration $t$ be of the form, ${g_{t} = \nabla F_{t} + \xi_{t},}$ with ${\mbE [\xi_{t}| x_{t}] ={0}}$. This leads to the ($M, \sigma^2$)-bounded noise assumption on the stochastic gradient \citep{stich2020error}. 
Among different assumptions on bounding the stochastic gradient \citep{doi:10.1137/16M1080173,ghadimi2013stochastic,stich2020error, Lei_TNNLS,gower21a, vaswani2019fast, layer-wise}, recently, \cite{khaled2020better} noted that the expected smoothness is the weakest among them and is as follows:

\begin{assumption}\label{ass:ABC}
\textbf{(Expected smoothness)} There exist constants $A,B,C\geq0$ such that for all $x_t\in\R^d$, we have $
   \textstyle \mbE[\norm{g_{t}}^2 \mid x_{t}] \leq 2A(F_t-\fs)+B\|\nabla F(x_t)\|^2+C.
$
\end{assumption}
Our analyses is based on Assumption \ref{ass:ABC} 
which contains other assumptions 
as special cases. 

\section{Analysis}\label{sec:concentartion}
The standard analysis of nonconvex convergence of SGD use the minimum of the expected gradient norm, ${\min_{t\in[T]}\mathbb{E}\|\nabla F(x_t)\|\to 0}$, or the average of the expected gradient norm, ${\frac{1}{T}\sum_t\mathbb{E}\|\nabla F(x_t)\|\to 0}$ as ${T\to\infty}$, by using different conditions on the second moment of the stochastic gradient. In this section, we first modify that result and show that the last iterations of SGD would maintain the same complexity guarantee under different stepsize choices but under the same set of assumptions. Finally, an interesting consequence of our convergence analyses is that they allow us to measure the \emph{density of the $\epsilon$-stationary points} in the final iterates of SGD without any additional assumptions. 
\begin{figure*}
    \centering
    \includegraphics[width=1\textwidth]{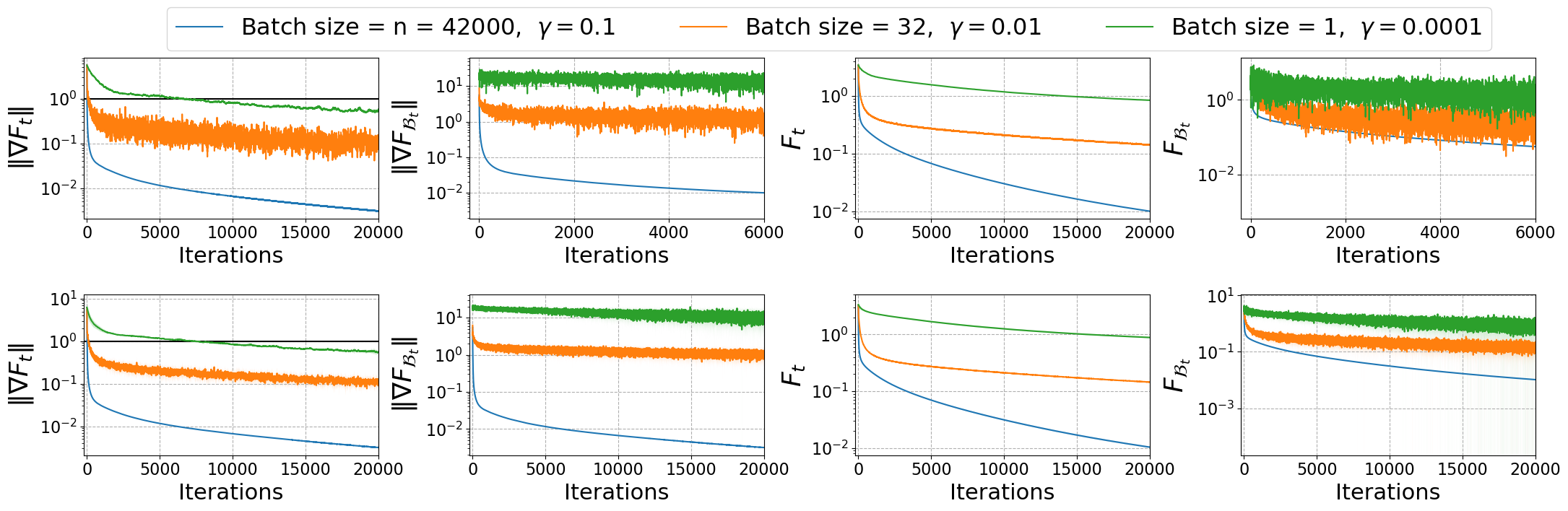}
    \caption{\small{Performance of \texttt{SGD} on MNIST digit classification. The top row shows the result of 1 single run of \texttt{SGD} while the bottom row shows the  result of the average of 10 runs. For the plots in the first column, the horizontal lines correspond to the precision, ${\epsilon=1}$---For SGD, if the total number of iterations is large enough then the entire tail comprises of the $\epsilon$-stationary points. These plots confirm our observation in Figure \ref{fig:log_logistic} but for more complicated models.}}\label{fig:N1}
\end{figure*}
Denote $r_t=\mbE\|\nabla F(x_t)\|^2$, $\delta_t=\mbE(F_t)-\fs$, and $D_{T}:=\prod_{t=1}^{T}(1+LA\gamma_{t}^{2})$. Further denote the set of indices of $\epsilon$-stationary points, $S_{\epsilon}\eqdef\{t:r_t\le \epsilon\}.$ For $\eta\in (0,1]$, let $S_{\epsilon,\eta}=S_{\epsilon}\cap [(1-\eta)T,T]$.
We start by quoting the general result for nonconvex convergence of SGD which is contained in the derivation of many known results in the literature. 

\begin{theorem}\label{thm:general_analysis_sgd}
Suppose Assumptions \ref{ass:minimum}, \ref{ass:smoothness}, and \ref{ass:ABC} hold. 
With the notations above, we have for any learning rate $(\gamma_{t})$ satisfying $\gamma_{t}\leq1/LB$ that
$
\sum_{t=1}^{T}\gamma_{t}r_{t-1}\leq \left(2\delta_{0}+\frac{C}{A}\right)D_{T}.
$
\end{theorem}
From Theorem \ref{thm:general_analysis_sgd}, we obtain the classic nonconvex convergence of SGD with a constant learning rate. For decreasing learning rate of Theorem~\ref{thm:general_analysis_sgd}, see Corollary~\ref{cor:general_analysis_sgd:decreasing}. 
\begin{corollary}[Constant learning rate] \citep{khaled2020better} \label{cor:general_analysis_sgd:constant}
For constant stepsize $\gamma_{t}=\gamma$ in Theorem \ref{thm:general_analysis_sgd}, we have
$\min_{1\leq t\leq T}r_{t} \leq \frac{D_{T}}{\gamma T}\left(2\delta_{0}+\frac{C}{A}\right) \leq \frac{3\sqrt{LA}}{\ln(3)\sqrt{T}}\left(2\delta_{0}+\frac{C}{A}\right).$
\end{corollary}
Corollary \ref{cor:general_analysis_sgd:constant} shows that in the entire range of iterates, $[T]$, there is one $r_{t}$ that approaches to 0 with a rate $O(\frac{1}{\sqrt{T}})$ as ${T\to\infty}$. But as we mentioned, this information is imprecise.

We now outline the proof of Theorem~\ref{thm:general_analysis_sgd}. 
Note that the key step using the $L$-smoothness of $F$ and {\it expected smoothness} of the stochastic gradients is the following inequality \cite[Lemma 2]{khaled2020better}:
\begin{equation}\label{eq:sgd ineq main-es}
\textstyle{\gamma_{t}\left(1-\frac{LB\gamma_{t}}{2}\right)r_{t} \leq \left(1+LA\gamma_{t}^{2}\right)\delta_{t} -\delta_{t+1}+\frac{LC\gamma_{t}^{2}}{2}}.
\end{equation}
Next, for any learning rate $(\gamma_{t})$ chosen such that ${\sum_{t=1}^{\infty}\gamma_{t}^{2}<\infty}$, we have ${D_{\infty}\leq\exp(LA\sum_{t=1}^{\infty}\gamma_{t}^{2})<\infty}$. Unrolling the recurrence relation in (\ref{eq:sgd ineq main-es}) reduces the LHS to
${\sum_{t=1}^{T}\prod_{i=t+1}^{T}(1+LA\gamma_{i}^{2})(\gamma_{t}-LB\gamma_{t}^{2}/2)r_{t-1}}$. If we assume ${\gamma_{t}\leq1/LB}$, then ${\gamma_{t}-LB\gamma_{t}^{2}/2\geq\gamma_{t}-\gamma_{t}/2=\gamma_{t}/2}$, and the lower bound of the LHS of the inequality (\ref{eq:sgd ineq main-es}) becomes
$\frac{1}{2}\sum_{t=1}^{T}\gamma_{t}r_{t-1},$ as $\prod_{i=1}^{T}(1+LA\gamma_{i}^{2})\geq 1$. Hence, the proof of Theorem \ref{thm:general_analysis_sgd}; see more detailed proof in \S B.1.

Can we do any better? The answer to this question demystifies the first myth of the nonconvex convergence of SGD.  

\textbf{Myth I: The nonconvex convergence of SGD is given by the minimum of the norm of the gradient function, $r_t,$ over the range of entire iterates $[T]$.} 
For any ${k \in [T]}$, we can split the LHS $\frac{1}{2}\sum_{t=1}^{T}\prod_{i=t+1}^{T}(1+LA\gamma_{i}^{2})\gamma_{t}r_{t-1}$ of Theorem~\ref{thm:general_analysis_sgd} into two sums as:
$\frac{1}{2}\sum_{t=1}^{k-1}\prod_{i=t+1}^{T}(1+LA\gamma_{i}^{2})\gamma_{t}r_{t-1} + \frac{1}{2}\sum_{t=k}^{T}\prod_{i=t+1}^{T}(1+LA\gamma_{i}^{2})\gamma_{t}r_{t-1}.
$
Manipulating these quantities, we can obtain more precise information on the convergence of the $r_t$ in the tail portion of the iterates of SGD for nonconvex functions, which is given by $\min_{k\leq t \leq T}r_{t}$.~To the best of our knowledge, this marks an improvement over the existing classical convergence results in \cite{ghadimi2013stochastic,khaled2020better,stich2020error}. We formally quote this in the following theorem; see proof in Appendix \S\ref{appendix:proofs}. 
\begin{theorem} \label{thm:general_analysis_sgd:k}
Suppose Assumptions \ref{ass:minimum}, \ref{ass:smoothness}, and \ref{ass:ABC} hold. 
For any learning rate $(\gamma_{t})$ satisfying ${\gamma_{t}\leq1/LB}$, we have for any ${k \in [T]}$,
${\min_{k\leq t \leq T}r_{t} \leq \left(2\delta_{0}+\frac{C}{A}-\frac{2\delta_{T}}{D_{T}}-\frac{1}{LA\gamma_{1}}\min_{1\leq t \leq k-1}r_{t}\right)LA\gamma_{k}D_{k}
\leq\left(2\delta_{0}+\frac{C}{A}\right)LA\gamma_{k}D_{k}.}$
\end{theorem}
In the subsequent analyses, we shed more light on the above result by considering different stepsizes. 


\textbf{Constant step-size $\gamma_t=\gamma$.}
Denote $D\eqdef (1+L\gamma^2A), E\eqdef \gamma(1-LB\gamma/2), F\eqdef\frac{L\gamma^2C}{2},$ and $W=\sum_{t=0}^TD^t.$ By using (\ref{eq:sgd ineq main-es}), the existing results show, $\min_{t\in[T]}r_t\to 0.$ First, we will improve this result. 

Unrolling the recurrence in (\ref{eq:sgd ineq main-es}), and using the above mentioned notations, we have:
\begin{equation}\label{eq:sgd ineq main-3star-es}
\sum_{t=0}^T{(1+L\gamma^2 A)^{T-t}}r_t +\frac{2\delta_{T+1}}{\gamma (2-L\gamma B)} 
\le \frac{2\delta_{0}(1+L\gamma^2A)^{T+1}}{\gamma(2-LB\gamma)}
+\frac{C[(1+L\gamma^2A)^{T+1}-1]}{\gamma A(2-LB\gamma)}.~\;\;
\end{equation}
Let $\eta\in (0,1]$. The 
LHS in the inequality (\ref{eq:sgd ineq main-3star-es}) is bounded from below by 
\begin{equation*}
\min_{(1-\eta)T\leq t\leq T} r_t
\sum_{(1-\eta)T\leq t\leq T} (1+L\gamma^2 A)^{T-t}
\geq (\eta T-1)\min_{(1-\eta)T\leq t\leq T} r_t;
\end{equation*}
if ${LB\gamma \leq 1}$ and ${(1+L\gamma^2 A)^{T+1}\leq 3}$ then the 
RHS of (\ref{eq:sgd ineq main-3star-es}) could be bounded from above by
\begin{equation}\label{upper-bound}
\frac{6\delta_0}{\gamma}+\frac{2C}{\gamma A}.
\end{equation}
Hence, we obtain
\begin{equation}\label{x1:sgd}
\min_{(1-\eta)T\leq t\leq T} r_t\leq 2\left(3{\delta_0}+\frac{C}{ A}\right)\frac{1}{(\eta T-1)\gamma}.
\end{equation}
Now, letting ${\gamma:=\sqrt{\frac{\ln 3}{(T+1)LA}},}$ we can show the following result; see Appendix \S\ref{Appendix:es_conv} for the proof.

\begin{theorem}\label{theorem:main-es}
Suppose Assumptions \ref{ass:minimum}, \ref{ass:smoothness}, and \ref{ass:ABC} hold. Let $\epsilon>0$ and $\eta\in (0,1]$. If the number of iterations $T\geq 1$ satisfies
$
T\geq \max \left\{\left(\frac{4\sqrt{2LA}(3\delta_0+C/A)}{\varepsilon \eta \sqrt{\ln 3} }\right)^2, \frac{LB^2\ln 3 }{A}-1, \frac{2}{\eta}\right\},
$
then, there exists an 
index ${t\geq (1-\eta)T}$ such that ${\mbE\|\nabla F(x_t)\|^2\le \epsilon.}$
\end{theorem}
By controlling the stepsize parameter, Theorem \ref{theorem:main-es} shows the $\epsilon$-stationary points exist in the final iterates of SGD for minimizing nonconvex functions.
\begin{remark}\label{remark:stationary}
Let ${\epsilon>0}.$ For choice of arbitrary ${\eta \in (0,1)}$, and ${T=\Omega\left(\max\{\frac{1}{\eta},\frac{1}{\eta^2\epsilon^2}\}\right),}$ there exists a ${t\in[(1-\eta)T,T]},$ such that, ${\mbE\|\nabla F(x_t)\|^2\le \epsilon}$. 
E.g., take $\eta=0.05$ in the Theorem above. Then we know that the last $5\%$ steps in the $T$ iterations will produce at least one $\varepsilon$-stationary point. 
For ${\eta=1}$ in Theorem \ref{theorem:main-es}, we recover the classical asymptotic convergence rate of SGD, that is, ${\min_{t\in[T]}\mbE\|\nabla F(x_t)\|^2 = O\left(\frac{1}{\sqrt{T}}\right).}$  
\end{remark}
\begin{remark}
Our choice of ${\gamma}$ makes the expression ${(1+L\gamma^2 A)^{T+1}}$ contained in ${[\sqrt{3},3]}$ to the right 
of $1$ on the real line. Any stepsize such that the expression is  contained in an interval on the right 
of $1$ will work, only with the difference in the constants in the estimations.
\end{remark}

\textbf{Decreasing step-size.}
We consider stepsize ${\gamma_t=\frac{\gamma_0}{\sqrt{t+1}}}$ with ${\gamma_0>0}$, and adopt a slightly different technique.~Inspired by \cite{stich2020error}, we define a non-negative, decreasing weighting sequence, $\{w_t\}_{t=0}^T$, such that $w_{-1}=1$ and  $w_t\eqdef\frac{w_{t-1}}{(1+L\gamma_{t}^2A)}.$ Note that, the weights do not appear in the convergence result. With these weights, and by using the notations before, we can rewrite (\ref{eq:sgd ineq main-es}) as:
\begin{equation*}\label{eq:sgd ineq main-ds-1}
w_t\gamma_t(1-\frac{LB\gamma_t}{2})r_t \le w_t(1+L\gamma_t^2A)\delta_t-w_t\delta_{t+1} +\frac{w_tL\gamma_t^2C}{2}.
\end{equation*}
Taking summation on above from $t=0$ to $t=T$, we have
\begin{eqnarray}\label{eq:sgd ineq main-ds-sum}
\sum_{t=0}^Tw_t\gamma_t(1-\frac{LB\gamma_t}{2})r_t\le \delta_0+\frac{LC}{2}\sum_{t=0}^Tw_t\gamma_t^2.
\end{eqnarray}
The 
RHS of (\ref{eq:sgd ineq main-ds-sum}) is bounded above by
\begin{eqnarray}\label{eq:upper_bd}
\delta_0+\frac{LC}{2}\gamma_0^2(\ln(T+1)+1).
\end{eqnarray}
Following the same technique as in the constant stepsize case, the 
LHS of (\ref{eq:sgd ineq main-ds-sum}) is bounded from below by 
\begin{eqnarray}\label{eq:lower_bd}
&(1-LA\gamma_0^2\ln(T+1))\min_{(1-\eta)T\leq t\leq T} r_t(\gamma_0(1-
\sqrt{1-\eta})\sqrt{T+1}-\frac{LB\gamma_0^2}{2}\ln(T+1)\nonumber\\
&+\frac{LB\gamma_0^2}{2}\ln([(1-\eta)T]+1)).
\end{eqnarray}
Combining (\ref{eq:upper_bd}) and (\ref{eq:lower_bd}), we can state the following Theorem; see Appendix \S\ref{Appendix:es_conv} for the proof.  
\begin{theorem}\label{theorem:main-es-decreasing} Suppose Assumptions \ref{ass:minimum}, \ref{ass:smoothness}, and \ref{ass:ABC} hold. Let ${\eta\in (0,1]}$. By choosing the stepsize $\gamma_t=\frac{\gamma_0}{\sqrt{t+1}}$ with  $\gamma_0^2<\frac{1}{LA\ln(T+1)}$, there exists a step $t\geq (1-\eta)T$ such that ${\mbE\|\nabla F(x_t)\|^2\le \frac{F_0-F_\star+\frac{LC\gamma_0^2}{2}(\ln(T+1)+1)}{(1-LA\gamma_0^2\ln(T+1))\cC(t)}},$
where 
$\cC(t) \eqdef (\gamma_0\eta\sqrt{T+1}-\frac{LB\gamma_0^2}{2}\ln(T+1) +\frac{LB\gamma_0^2}{2}\ln([(1-\eta)T]+1)).$
\end{theorem}
\begin{remark}\label{remark:stationary}
From the previous Theorem we have for all fixed  $\eta$ in $(0,1]$ there exists a ${t\in[(1-\eta)T,T]},$ such that ${\mbE\|\nabla F(x_t)\|^2 = 
O\left(\frac{\ln(T+1)}{\sqrt{T+1}}\right).}$
For ${\eta=1}$, we recover the classical asymptotic convergence rate of SGD.
\end{remark}
This result is also an improvement over the existing nonconvex convergence results of SGD for decreasing stepsize \citep{khaled2020better, stich2020error}. 
See Corollary \ref{cor:general_analysis_sgd:decreasing:k} in \S\ref{appendix:proofs} for decreasing stepsize of the form $\gamma_{t}=\gamma t^{-\alpha}$ with $\alpha\in(1/2,1).$ One could think of other stepsizes in Theorem~\ref{thm:general_analysis_sgd:k}, e.g., $O(1/\sqrt{t})$, $O(1/\ln(t)\sqrt{t}),$ and so on, and find the same guarantee.

\begin{myblock}{Open questions}
\small{\myNum{i}~For SGD, without an iteration budget,~$T$, are there practical ways to detect the first iterate,~$t$, for which ${\E{\|\nabla{f}(x_t)\|^2}\leq\epsilon}?$ \myNum{ii}~How to terminate SGD other than the maximum number of iterations?}
\end{myblock}

\textbf{Legend I: As long as (\ref{eq:sgd ineq main-es}) is used as a key descent inequality, there is no better convergence result.} Recall (\ref{eq:sgd ineq main-es}) is the key inequality to prove the convergence of SGD for nonconvex and $L$-smooth functions; see Theorem \ref{thm:general_analysis_sgd}. We modified (\ref{eq:sgd ineq main-es}), and give the convergence of SGD in the tail for $\eta T$ iterations with ${\eta \in (0,1)}$. For ${\eta \in (0,1)}$, one could think of taking $k$ in specific form in Theorem~\ref{thm:general_analysis_sgd:k}, e.g., $k=(1-\eta)T$, $k=(1-1/\eta)T$, $k=T-\eta\ln(T)$ or $k=T-\ln(T),$ and get a $O(\frac{1}{\sqrt{T}})$ convergence rate for the tail. However, if $k=T-\cC$, where $\cC$ is a constant, the rate disappears. Therefore, we must design a better descent inequality than (\ref{eq:sgd ineq main-es}). 

\textbf{Myth II: Different stepsize choices result in novel convergence guarantees for SGD.}
Several works propose a plethora of stepsize choices to boost the convergence of SGD \citep{gower21a, loizou2021stochastic, schaipp2023a}; \citep{wang2021convergence} provide nonconvex convergence results for step decay. 
Since the quantity for convergence and the key descent step is problematic, different stepsize choices result only in incremental improvement other than a giant leap.  

\textbf{Myth III: Better assumptions in upper bounding the stochastic gradients result in better convergence.}
Recently, \cite{allen2019convergence} showed why first-order methods such as SGD can find global minima on the training objective of multi-layer DNNs with ReLU activation and almost surely achieves close to zero training error in polynomial time. It was possible because the authors proved two instrumental properties of the loss, $F$: \myNum{i} gradient
bounds for points that are sufficiently close to the random
initialization, and \myNum{ii} semi-smoothness. 

Assumptions on the bounds of the stochastic gradient are important, but this work does not judge which is better than the other as the literature has established potentially many interplays between them. Nevertheless, no significant improvements are possible following the present approach unless we propose completely new Assumptions, which should also be practical; but they are not easy to make. New assumptions on the upper bound of the stochastic gradients would only lead to minor improvements. 

\begin{myblock}{Open questions}
\cite{yu2021analysis} assume the dissipativity property of the loss function, $F$ (less restrictive than convexity), and obtained guarantees on the last iterate of SGD. \myNum{i} What is the biggest class of nonconvex functions for which we can still have such guarantees? \myNum{ii} What are the minimal additional assumptions on $F$ to ensure the convergence of the last iterate?
\end{myblock}


\textbf{Legend II: If SGD is run for many iterates, then the tail will {\em almost surely} produce $\epsilon$-stationary points.}~It is almost impossible to prove the last iterate convergence of SGD for nonconvex functions (see discussion in \cite{dieuleveut2023stochastic}), but empirical evidence tells that the stationary points are mostly concentrated at the tail \citep{shalev2011pegasos}.~Nevertheless, how can we formalize it? The central idea in Theorem \ref{theorem:main-es} is to bound a weighted sum of the gradient norms over all iterations from below by a partial sum over the last $\eta T$ iterations. Moreover, the result is not only mathematically stronger than the existing ones, but the simple trick of the partial sum over the last $\eta T$ iterations leads us to another significant result. 

We know $S_{\epsilon,\eta}\not = \emptyset$ by Theorem~\ref{theorem:main-es}.~On one hand, we have
\begin{eqnarray}\label{eq:e_st-es}
&&{\sum_{t=(1-\eta)T}^T(1+L\gamma^2 A)^{T-t}r_t>\sum_{t\in S_{\epsilon}^c\atop t\geq (1-\eta) T}(1+L\gamma^2 A)^{T-t}r_t >\sum_{t = (1-\eta)T + |S_{\epsilon, \eta}|}^T(1+L\gamma^2 A)^{T-t}\epsilon} \nonumber\\
&\geq&{\frac{(1+L\gamma^2 A)^{\eta T - |S_{\epsilon, \eta}| }-1}{L\gamma^2 A}
  \epsilon},
\end{eqnarray}
where $|S_{\epsilon,\eta}|$ denotes the cardinality of the set $S_{\epsilon, \eta}.$
Note that, $\sum_{t\in S_{\epsilon}^c\atop t\geq (1-\eta) T}(1+L\gamma^2 A)^{T-t}r_t$ has $(\eta T-|S_{\epsilon,\eta}|+1)$ terms; so, we lower bound them with the smallest of those many terms. 
On the other hand, using (\ref{eq:sgd ineq main-3star-es}) and (\ref{upper-bound}) to bound the left hand side of (\ref{eq:e_st-es}), and rearranging the terms we obtain
$$(1+L\gamma^2 A)^{\eta T - |S_{\epsilon, \eta}| }
  \leq \frac{{6\delta_0L\gamma A}+{2CL\gamma}}{\epsilon}+1.
$$
Taking logarithm to the previous inequality and rearranging the terms we get the following theorem:
\begin{theorem}\label{theorem:density}
Suppose Assumptions \ref{ass:minimum}, \ref{ass:smoothness}, and \ref{ass:ABC} hold. Let $\epsilon>0$ and $\eta\in (0,1]$. For constant stepsize $\gamma_{t}=\gamma$, we have
$
\frac{|S_{\epsilon, \eta}|}{\eta T}\geq 1-
\frac{1}{ T}\frac{\ln\left(\frac{{6\delta_0L\gamma A}+{2CL\gamma}}{\epsilon}+1 \right)}{\eta\ln\left( 1+L\gamma^2 A \right)}.
$
\end{theorem}
From the previous Theorem, we see that the density of the $\epsilon$-stationary points in the top $\eta$ portion of the tails approaches $1$ as $T$ increases, which roughly speaking, tells us that for $T$ large enough almost all the iterations $x_t$ for $t\in [(1-\eta)T, T]$ are $\epsilon$-stationary points.

Recall, from Theorem \ref{theorem:main-es}, if the total number of iterations, $T$ is large enough then there will be iterate, ${t\in[(1-\eta)T,T]}$ to produce ${\mbE\|\nabla F_{t}\|^2\le \epsilon},$ where ${\eta\in(0,1].}$ That is, Theorem \ref{theorem:main-es} guarantees the existence of (at least one) stationary point(s) in the final iterates. However, one can argue that the effect of $\eta$  in the complexity can make it worse when $\eta$ is small, which requires the SGD to run for a sufficiently large number of iterations. Whereas the claim ${\frac{|S_{\epsilon, \eta}|}{\eta T}\to 1}$ as ${T\to\infty}$, says that running SGD for a large number of iterations, $T$ is not necessarily problematic, as almost surely the density of the stationary points in the tail portion will approach to 1, guaranteeing the entire tail comprising mostly of $\epsilon$-stationary points; see Figure \ref{fig:log_logistic-concentration}. 

Similarly, for the decreasing stepsize case, from Theorem~\ref{theorem:main-es-decreasing}, we have $S_{\epsilon,\eta}\not = \emptyset$. For $T$ large enough, we can lower bound the left side of the inequality (\ref{eq:sgd ineq main-ds-sum}) as 
\begin{eqnarray*}
&&\sum_{t=0}^T w_t\gamma_t(1-\frac{LB\gamma_t}{2})r_t
\ge \epsilon w_T\sum_{t\in S_{\epsilon}^c\atop t \geq (1-\eta) T}\left(\gamma_t-\frac{LB\gamma_t^2}{2}\right)
\geq  \epsilon (1+L\gamma_0^2A(T+1))(\gamma_0\sqrt{T+1}\\
&&-\gamma_0\sqrt{(1-\eta)T+|S_{\epsilon,\eta}|}-\frac{LB\gamma_0^2}{2}\ln(T+1)).
\end{eqnarray*}
The above, combined with the upper bound in (\ref{eq:upper_bd}) can be written as 
\begin{eqnarray*}
\gamma_0(\sqrt{T+1}-\sqrt{(1-\eta)T+|S_{\epsilon,\eta}|})
\le \underbrace{\frac{\delta_0+\frac{LC}{2}\gamma_0^2(\ln(T+1)+1)}{(1+L\gamma_0^2A(T+1))}+\frac{LB\gamma_0^2}{2}\ln(T+1)}_{\eqdef \cD},
\end{eqnarray*}
which can be further reduced to 
\begin{eqnarray}
\frac{|S_{\epsilon, \eta}|}{\eta T}\geq 1-2\frac{\cD}{\gamma_0\eta \sqrt{T}}+\frac{\cD^2}{\gamma_0^2\eta T}.
\end{eqnarray}
Formally, we summarize it in the following theorem.  
\begin{theorem}\label{theorem:density-dss}
Suppose Assumptions \ref{ass:minimum}, \ref{ass:smoothness}, and \ref{ass:ABC} hold. Let $\epsilon>0$ and $\eta\in (0,1]$. For decreasing step size, ${\gamma_t=\frac{\gamma_0}{\sqrt{t+1}}}$ with ${\gamma_0>0}$, we have
$
\frac{|S_{\epsilon, \eta}|}{\eta T}\geq 1-O\left(\frac{1}{ \sqrt{T}}\right).
$
\end{theorem}
Similar to the argument for constant stepsize case, we conclude that the density of the $\epsilon$-stationary points in the top $\eta$ portion of the tail approaches to $1$ as $T$ increases. 
\begin{remark}\label{remark:decreasing_stationary}
 For stepsizes of the form $\gamma_{t}=\gamma t^{-\alpha}$ with $\alpha\in(1/2,1),$ and $\gamma\in\R^+$, or other stepsizes, e.g., $O(1/\sqrt{t})$, $O(1/\ln(t)\sqrt{t}),$ and so on, we can find similar guarantee.   
\end{remark}


\textbf{Connection with the high probability bounds.} The {\em high probability} convergence results, first used by \cite{kakadetiwari2008}, and then in \cite{harvey19a, harvey2019simple}, proved convergence bounds for different class of loss functions (e.g., Lipschitz and strongly convex, but not necessarily differentiable; Lipschitz and convex, but not necessarily strongly convex or differentiable) by using rate Martingale and generalized Freeman's inequality. For Lipschitz and strongly convex functions, but not necessarily differentiable, \cite{harvey19a} proved that after $T$ steps of SGD, the error of the final iterate is $O(\log(T)/T)$ with high probability; also, see \cite{de2015global} for global convergence of SGD for matrix completion and related problems. \cite{cutkosky2021high} provided high-probability bounds for nonconvex stochastic optimization with heavy tails.~We refer to~\cite{arjevani_lowerbound} and references therein for the lower bound on the complexity of finding an $\epsilon$-stationary point using stochastic first-order methods. \cite{harvey19a} mentioned ``high probability bounds are more difficult," as it controls the total noise of the noisy (sub)-gradients at each iteration, and not obvious that the error of the final iterate is tightly concentrated. To our knowledge, there is no trivial analysis. In contrast, we show the concentration of $\epsilon$-stationary point, ${{|S_{\epsilon, \eta}|}/{\eta T}}$ in the final iterates of nonconvex SGD as $T$ increases, approaches to 1, using simple arguments. Similar to \cite{harvey19a}, our result can be generalized for SGD with suffix averaging or weighted averaging. 


\begin{figure}
    \centering
    \includegraphics[width=0.7\textwidth]{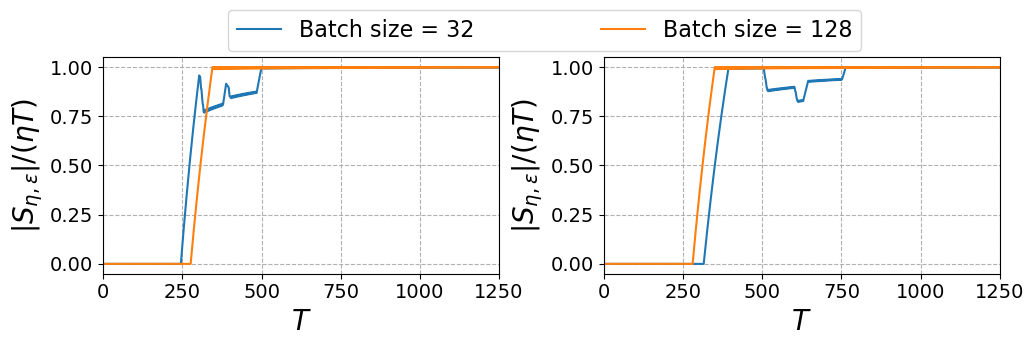}
    \caption{\small{Concentration of the $\epsilon$-stationary points, ${{|S_{\epsilon, \eta}|}/{\eta T}}$ vs. Iterations for nonconvex logistic regression problem running SGD (left) and {\em random reshuffling SGD} (RR-SGD) (right) for choice of ${\epsilon=10^{-2}}$ and ${\eta=0.2}$. 
    See theoretical convergence result of RR-SGD in Appendix \ref{Appendix:rrsgd_conv}. As $T$ increases, ${{|S_{\epsilon, \eta}|}/{\eta T}\to1}$ from below. The tail will {\em almost surely} produce $\epsilon$-stationary points and their density in the tail will approach to 1---It confirms our theoretical finding in Theorem \ref{theorem:density}.}}
    \label{fig:log_logistic-concentration}
\end{figure}

\section{Numerical evidence}\label{sec:numerical results}
We conduct experiments on nonconvex functions with $L$-smooth and non-smooth (for DNNs) loss to substantiate our theoretical results that are based on stochastic gradient, $g_t.$ In practice, $g_t$ can be calculated by sampling and processing minibatches of data. Therefore, besides $\| \nabla F(x_t) \|$ and $ F(x_t)$, we also track, the norm of the minibatch stochastic gradient, $\| \nabla F_{\mathcal{B}_t} \|$, and minibatch stochastic loss, $F_{\mathcal{B}_t}.$ 
Note that, $\mathcal{B}_t$ is the selected minibatch of data at iteration $t$ and $F_{\mathcal{B}_t}:= \frac{1}{|\mathcal{B}_t|}\sum_{i \in \mathcal{B}_t} f_i(x_t).$ 
\textbf{Nonconvex and $L$-smooth loss.} We consider logistic regression with nonconvex regularization:
$$
{\min_{x \in \mathbb{R}^{d}}\left[F(x)\eqdef\frac{1}{n} \sum_{i=1}^{n} \ln (1+\exp (-a_i^\top x))  + \lambda \sum_{j=1}^{d} \frac{x_j^2}{1+x_j^2}\right]},
$$
where $a_1,a_2,...,a_n \in \mathbb{R}^d$ are the given data, and $\lambda>0$ is the regularization parameter. We run the experiments on the \texttt{a1a} dataset from LIBSVM \citep{libsvm}, where ${n=1605, d=123,}$ and set $\lambda=0.5$. Figures \ref{fig:log_logistic} (and Figure \ref{fig:logistic_rrsgd} in \S\ref{Appendix:rrsgd_conv} for \texttt{RR-SGD}) shows the average of 10 runs of \texttt{SGD} with different minibatch sizes. The shaded area is given by $\pm \sigma$ where $\sigma$ is the standard deviation. 

\textbf{Nonconvex and nonsmooth loss.} We use a feed forward neural network (FNN) for MNIST digit \citep{lecun1998gradient_mnist} classification. The FNN has one hidden layer with 256 neurons activated by ReLU, and an 10 dimensional output layer activated by the softmax function. The loss function is the categorical cross entropy. We calculate the loss and the stochastic gradient during the training by using different minibatches. The entire loss and the full gradient are computed using all ${n=42\times 10^3}$ samples. For the average of 10 runs, the shaded area is given by $\pm \sigma,$ where ${\sigma\geq 0}$ is the standard deviation, and $\gamma$ is the learning rate. Figures \ref{fig:N1} in the main paper (and Figure \ref{fig:N2} in \S\ref{Appendix:rrsgd_conv}) show if the total number of iterations is large enough then the entire tail comprising of the $\epsilon$-stationary points.

\textbf{Concentration of $\epsilon$-stationary points.} For ${\epsilon=10^{-2}}$ and ${\eta=0.2}$, in Figure \ref{fig:log_logistic-concentration}, we plot the density of the $\epsilon$-stationary points, ${{|S_{\epsilon, \eta}|}/{\eta T}}$ as a function of iteration, $T$ for nonconvex logistic regression problems. As $T$ increases,  ${{|S_{\epsilon, \eta}|}/{\eta T}\to1}$ from below, and conform our result in \S\ref{sec:concentartion}.

\section{Conclusion}
To the best of our knowledge, this is the first work that proves the existence of {\em an $\epsilon$-stationary point in the final iterates of SGD in optimizing nonconvex functions}, given a large enough total iteration budget, $T$, and {\em measures the concentration of such $\epsilon$-stationary points} in the final iterates.~Alongside, by using simple reasons, we informally addressed certain myths and legends that determine the practical success of SGD in the nonconvex world and posed some new research questions.

\bibliography{references}

\begin{thebibliography}{67}
\providecommand{\natexlab}[1]{#1}
\providecommand{\url}[1]{\texttt{#1}}
\expandafter\ifx\csname urlstyle\endcsname\relax
  \providecommand{\doi}[1]{doi: #1}\else
  \providecommand{\doi}{doi: \begingroup \urlstyle{rm}\Url}\fi

\bibitem[Alistarh et~al.(2017)Alistarh, Grubic, Li, Tomioka, and
  Vojnovic]{alistarh2017qsgd}
Dan Alistarh, Demjan Grubic, Jerry Li, Ryota Tomioka, and Milan Vojnovic.
\newblock {QSGD: Communication-efficient SGD via gradient quantization and
  encoding}.
\newblock In \emph{Advances in Neural Information Processing System}, pp.\
  1709--1720, 2017.

\bibitem[Allen-Zhu et~al.(2019)Allen-Zhu, Li, and Song]{allen2019convergence}
Zeyuan Allen-Zhu, Yuanzhi Li, and Zhao Song.
\newblock A convergence theory for deep learning via over-parameterization.
\newblock In \emph{International Conference on Machine Learning}, pp.\
  242--252, 2019.

\bibitem[Arjevani et~al.(2019)Arjevani, Carmon, Duchi, Foster, Srebro, and
  Woodworth]{arjevani_lowerbound}
Yossi Arjevani, Yair Carmon, John Duchi, Dylan Foster, Nathan Srebro, and Blake
  Woodworth.
\newblock Lower bounds for non-convex stochastic optimization.
\newblock \emph{CoRR}, abs/1912.02365, 2019.
\newblock URL \url{http://arxiv.org/abs/1912.02365}.

\bibitem[Bengio(2012)]{bengio2012practical}
Yoshua Bengio.
\newblock Practical recommendations for gradient-based training of deep
  architectures.
\newblock In \emph{Neural networks: Tricks of the trade}, pp.\  437--478.
  Springer, 2012.

\bibitem[Bottou et~al.(2018)Bottou, Curtis, and
  Nocedal]{doi:10.1137/16M1080173}
Léon Bottou, Frank Curtis, and Jorge Nocedal.
\newblock Optimization methods for large-scale machine learning.
\newblock \emph{SIAM Review}, 60\penalty0 (2):\penalty0 223--311, 2018.

\bibitem[Bousquet \& Elisseeff(2002)Bousquet and
  Elisseeff]{bousquet2002stability}
Olivier Bousquet and Andr{\'e} Elisseeff.
\newblock Stability and generalization.
\newblock \emph{The Journal of Machine Learning Research}, 2:\penalty0
  499--526, 2002.

\bibitem[Carmon et~al.(2020)Carmon, Duchi, Hinder, and
  Sidford]{carmon2020lower}
Yair Carmon, John Duchi, Oliver Hinder, and Aaron Sidford.
\newblock Lower bounds for finding stationary points i.
\newblock \emph{Mathematical Programming}, 184\penalty0 (1):\penalty0 71--120,
  2020.

\bibitem[Chang \& Lin(2011)Chang and Lin]{libsvm}
Chih-Chung Chang and Chih-Jen Lin.
\newblock Libsvm: a library for support vector machines.
\newblock \emph{ACM transactions on intelligent systems and technology},
  2\penalty0 (3):\penalty0 1--27, 2011.

\bibitem[Cutkosky \& Mehta(2021)Cutkosky and Mehta]{cutkosky2021high}
Ashok Cutkosky and Harsh Mehta.
\newblock High-probability bounds for non-convex stochastic optimization with
  heavy tails.
\newblock In \emph{Advances in Neural Information Processing Systems},
  volume~34, pp.\  4883--4895, 2021.

\bibitem[De~Sa et~al.(2015)De~Sa, Re, and Olukotun]{de2015global}
Christopher De~Sa, Christopher Re, and Kunle Olukotun.
\newblock Global convergence of stochastic gradient descent for some non-convex
  matrix problems.
\newblock In \emph{International Conference on Machine Learning}, pp.\
  2332--2341, 2015.

\bibitem[D{\'e}fossez et~al.(2020)D{\'e}fossez, Bottou, Bach, and
  Usunier]{defossez2020simple}
Alexandre D{\'e}fossez, L{\'e}on Bottou, Francis Bach, and Nicolas Usunier.
\newblock {A simple convergence proof of ADAM and AdaGrad}.
\newblock \emph{arXiv preprint arXiv:2003.02395}, 2020.

\bibitem[Deng et~al.(2009)Deng, Dong, Socher, Li, Li, and Fei-Fei]{imagenet}
Jia Deng, Wei Dong, Richard Socher, Li-Jia Li, Kai Li, and Li~Fei-Fei.
\newblock Imagenet: A large-scale hierarchical image database.
\newblock In \emph{Proceedings of the IEEE Conference on Computer Vision and
  Pattern Recognition}, pp.\  248--255, 2009.

\bibitem[Dieuleveut et~al.(2023)Dieuleveut, Fort, Moulines, and
  Wai]{dieuleveut2023stochastic}
Aymeric Dieuleveut, Gersende Fort, Eric Moulines, and Hoi-To Wai.
\newblock Stochastic approximation beyond gradient for signal processing and
  machine learning.
\newblock \emph{arXiv preprint arXiv:2302.11147}, 2023.

\bibitem[Duchi et~al.(2011)Duchi, Hazan, and Singer]{duchi2011adaptive}
John Duchi, Elad Hazan, and Yoram Singer.
\newblock Adaptive subgradient methods for online learning and stochastic
  optimization.
\newblock \emph{Journal of machine learning research}, 12\penalty0 (7), 2011.

\bibitem[Dutta et~al.(2020)Dutta, Bergou, Abdelmoniem, Ho, Sahu, Canini, and
  Kalnis]{layer-wise}
Aritra Dutta, ElHoucine Bergou, Ahmed Abdelmoniem, Chen-Yu Ho, Atal Sahu, Marco
  Canini, and Panos Kalnis.
\newblock {On the Discrepancy between the Theoretical Analysis and Practical
  Implementations of Compressed Communication for Distributed Deep Learning}.
\newblock In \emph{Proc. of AAAI}, volume~34, pp.\  3817--3824, 2020.

\bibitem[Fehrman et~al.(2020)Fehrman, Gess, and
  Jentzen]{fehrman2020convergence}
Benjamin Fehrman, Benjamin Gess, and Arnulf Jentzen.
\newblock Convergence rates for the stochastic gradient descent method for
  non-convex objective functions.
\newblock \emph{Journal of Machine Learning Research}, 21:\penalty0 136, 2020.

\bibitem[Feldman \& Vondrak(2018)Feldman and
  Vondrak]{feldman2018generalization}
Vitaly Feldman and Jan Vondrak.
\newblock Generalization bounds for uniformly stable algorithms.
\newblock In \emph{Advances in Neural Information Processing Systems},
  volume~31, 2018.

\bibitem[Ghadimi \& Lan(2013)Ghadimi and Lan]{ghadimi2013stochastic}
Saeed Ghadimi and Guanghui Lan.
\newblock Stochastic first-and zeroth-order methods for nonconvex stochastic
  programming.
\newblock \emph{SIAM Journal on Optimization}, 23\penalty0 (4):\penalty0
  2341--2368, 2013.

\bibitem[Ghadimi et~al.(2016)Ghadimi, Lan, and Zhang]{ghadimi2016mini}
Saeed Ghadimi, Guanghui Lan, and Hongchao Zhang.
\newblock Mini-batch stochastic approximation methods for nonconvex stochastic
  composite optimization.
\newblock \emph{Mathematical Programming}, 155\penalty0 (1-2):\penalty0
  267--305, 2016.

\bibitem[Gorbunov et~al.(2021)Gorbunov, Loizou, and
  Gidel]{gorbunov2021extragradient}
Eduard Gorbunov, Nicolas Loizou, and Gauthier Gidel.
\newblock Extragradient method: $o(1/k)$ last-iterate convergence for monotone
  variational inequalities and connections with cocoercivity.
\newblock \emph{International Conference on Artificial Intelligence and
  Statistics}, 151:\penalty0 366--402, 2021.

\bibitem[Gower et~al.(2019)Gower, Loizou, Qian, Sailanbayev, Shulgin, and
  Richt{\'a}rik]{pmlr-v97-qian19b}
Robert Gower, Nicolas Loizou, Xun Qian, Alibek Sailanbayev, Egor Shulgin, and
  Peter Richt{\'a}rik.
\newblock {SGD}: General analysis and improved rates.
\newblock In \emph{Proceedings of the 36th International Conference on Machine
  Learning}, volume~97, pp.\  5200--5209, 2019.

\bibitem[Gower et~al.(2021)Gower, Sebbouh, and Loizou]{gower21a}
Robert Gower, Othmane Sebbouh, and Nicolas Loizou.
\newblock {SGD for Structured Nonconvex Functions: Learning Rates, Minibatching
  and Interpolation }.
\newblock In \emph{International Conference on Artificial Intelligence and
  Statistics}, volume 130, pp.\  1315--1323, 2021.

\bibitem[Gratton et~al.(2008)Gratton, Sartenaer, and
  Toint]{gratton2008recursive}
Serge Gratton, Annick Sartenaer, and Philippe Toint.
\newblock Recursive trust-region methods for multiscale nonlinear optimization.
\newblock \emph{SIAM Journal on Optimization}, 19\penalty0 (1):\penalty0
  414--444, 2008.

\bibitem[G{\"u}rb{\"u}zbalaban et~al.(2021)G{\"u}rb{\"u}zbalaban, Ozdaglar, and
  Parrilo]{gurbuzbalaban2021random}
Mert G{\"u}rb{\"u}zbalaban, Asu Ozdaglar, and Pablo Parrilo.
\newblock Why random reshuffling beats stochastic gradient descent.
\newblock \emph{Mathematical Programming}, 186\penalty0 (1):\penalty0 49--84,
  2021.

\bibitem[Hardt et~al.(2016)Hardt, Recht, and Singer]{hardt2016train}
Moritz Hardt, Ben Recht, and Yoram Singer.
\newblock Train faster, generalize better: Stability of stochastic gradient
  descent.
\newblock In \emph{International Conference on Machine Learning}, pp.\
  1225--1234. PMLR, 2016.

\bibitem[Harvey et~al.(2019{\natexlab{a}})Harvey, Liaw, Plan, and
  Randhawa]{harvey19a}
Nicholas Harvey, Christopher Liaw, Yaniv Plan, and Sikander Randhawa.
\newblock Tight analyses for non-smooth stochastic gradient descent.
\newblock In \emph{Conference on Learning Theory}, volume~99, pp.\  1579--1613,
  2019{\natexlab{a}}.

\bibitem[Harvey et~al.(2019{\natexlab{b}})Harvey, Liaw, and
  Randhawa]{harvey2019simple}
Nicholas Harvey, Christopher Liaw, and Sikander Randhawa.
\newblock Simple and optimal high-probability bounds for strongly-convex
  stochastic gradient descent.
\newblock \emph{arXiv preprint arXiv:1909.00843}, 2019{\natexlab{b}}.

\bibitem[He et~al.(2016)He, Zhang, Ren, and Sun]{he2016deep}
Kaiming He, Xiangyu Zhang, Shaoqing Ren, and Jian Sun.
\newblock Deep residual learning for image recognition.
\newblock In \emph{Proceedings of the IEEE conference on computer vision and
  pattern recognition}, pp.\  770--778, 2016.

\bibitem[Jain et~al.(2019)Jain, Nagaraj, and Netrapalli]{pmlr-v99-jain19a}
Prateek Jain, Dheeraj Nagaraj, and Praneeth Netrapalli.
\newblock Making the last iterate of sgd information theoretically optimal.
\newblock In \emph{Conference on Learning Theory}, volume~99, pp.\  1752--1755,
  2019.

\bibitem[Jin et~al.(2022)Jin, Xing, and He]{jin2022convergence}
Ruinan Jin, Yu~Xing, and Xingkang He.
\newblock On the convergence of msgd and adagrad for stochastic optimization.
\newblock In \emph{International Conference on Learning Representations}, 2022.

\bibitem[Kairouz et~al.(2021)Kairouz, McMahan, Avent, Bellet, Bennis, Bhagoji,
  Bonawitz, Charles, Cormode, Cummings, et~al.]{kairouz2019advances}
Peter Kairouz, Brendan McMahan, Brendan Avent, Aur{\'e}lien Bellet, Mehdi
  Bennis, Arjun Bhagoji, Kallista Bonawitz, Zachary Charles, Graham Cormode,
  Rachel Cummings, et~al.
\newblock Advances and open problems in federated learning.
\newblock \emph{Foundations and Trends{\textregistered} in Machine Learning},
  14\penalty0 (1--2):\penalty0 1--210, 2021.

\bibitem[Kakade \& Tewari(2008)Kakade and Tewari]{kakadetiwari2008}
Sham Kakade and Ambuj Tewari.
\newblock On the generalization ability of online strongly convex programming
  algorithms.
\newblock In \emph{Advances in Neural Information Processing Systems},
  volume~21, 2008.

\bibitem[Khaled \& Richt{\'a}rik(2022)Khaled and
  Richt{\'a}rik]{khaled2020better}
Ahmed Khaled and Peter Richt{\'a}rik.
\newblock {Better theory for SGD in the nonconvex world}.
\newblock \emph{Transactions on Machine Learning Research}, 2022.

\bibitem[Kingma \& Ba(2015)Kingma and Ba]{kingma2015adam}
Diederik Kingma and Jimmy Ba.
\newblock Adam: A method for stochastic optimization.
\newblock In \emph{International Conference on Learning Representataion}, 2015.

\bibitem[Konečný et~al.(2016)Konečný, McMahan, Yu, Richtarik, Suresh, and
  Bacon]{FL:Jakub}
Jakub Konečný, Brendan McMahan, Felix Yu, Peter Richtarik, Ananda Suresh, and
  Dave Bacon.
\newblock Federated learning: Strategies for improving communication
  efficiency.
\newblock In \emph{Proc. of NeurIPS Workshop on Private Multi-Party Machine
  Learning}, 2016.

\bibitem[Korpelevich(1976)]{korpelevich1976extragradient}
Galina Korpelevich.
\newblock The extragradient method for finding saddle points and other
  problems.
\newblock \emph{Matecon}, 12:\penalty0 747--756, 1976.

\bibitem[LeCun et~al.(1998)LeCun, Bottou, Bengio, and
  Haffner]{lecun1998gradient_mnist}
Yann LeCun, L{\'e}on Bottou, Yoshua Bengio, and Patrick Haffner.
\newblock Gradient-based learning applied to document recognition.
\newblock \emph{Proceedings of the IEEE}, 86\penalty0 (11):\penalty0
  2278--2324, 1998.

\bibitem[Lei \& Ying(2020)Lei and Ying]{lei2020fine}
Yunwen Lei and Yiming Ying.
\newblock Fine-grained analysis of stability and generalization for stochastic
  gradient descent.
\newblock In \emph{International Conference on Machine Learning}, pp.\
  5809--5819, 2020.

\bibitem[Lei et~al.(2020)Lei, Hu, Li, and Tang]{Lei_TNNLS}
Yunwen Lei, Ting Hu, Guiying Li, and Ke~Tang.
\newblock Stochastic gradient descent for nonconvex learning without bounded
  gradient assumptions.
\newblock \emph{IEEE Transactions on Neural Networks and Learning Systems},
  31\penalty0 (10):\penalty0 4394--4400, 2020.

\bibitem[Li \& Li(2018)Li and Li]{li2018simple}
Zhize Li and Jian Li.
\newblock A simple proximal stochastic gradient method for nonsmooth nonconvex
  optimization.
\newblock \emph{Advances in neural information processing systems}, 31, 2018.

\bibitem[Loizou et~al.(2021)Loizou, Vaswani, Laradji, and
  Lacoste-Julien]{loizou2021stochastic}
Nicolas Loizou, Sharan Vaswani, Issam Laradji, and Simon Lacoste-Julien.
\newblock Stochastic polyak step-size for sgd: An adaptive learning rate for
  fast convergence.
\newblock In \emph{International Conference on Artificial Intelligence and
  Statistics}, pp.\  1306--1314, 2021.

\bibitem[Mishchenko et~al.(2020)Mishchenko, Khaled Ragab~Bayoumi, and
  Richt{\'a}rik]{mishchenko2020random}
Konstantin Mishchenko, Ahmed Khaled Ragab~Bayoumi, and Peter Richt{\'a}rik.
\newblock Random reshuffling: Simple analysis with vast improvements.
\newblock In \emph{Advances in Neural Information Processing Systems},
  volume~33, pp.\  17309--17320, 2020.

\bibitem[Needell et~al.(2014)Needell, Ward, and Srebro]{needell2014stochastic}
Deanna Needell, Rachel Ward, and Nati Srebro.
\newblock Stochastic gradient descent, weighted sampling, and the randomized
  kaczmarz algorithm.
\newblock In \emph{Advances in neural information processing systems},
  volume~27, 2014.

\bibitem[Nesterov(2003)]{nesterov2003introductory}
Yurii Nesterov.
\newblock \emph{Introductory lectures on convex optimization: A basic course},
  volume~87.
\newblock Springer Science \& Business Media, 2003.

\bibitem[Nguyen et~al.(2021)Nguyen, Tran-Dinh, Phan, Nguyen, and van
  Dijk]{nguyen2021unified}
Lam Nguyen, Quoc Tran-Dinh, Dzung Phan, Phuong Nguyen, and Marten van Dijk.
\newblock A unified convergence analysis for shuffling-type gradient methods.
\newblock \emph{Journal of Machine Learning Research}, 22\penalty0
  (207):\penalty0 1--44, 2021.

\bibitem[Popov(1980)]{popov1980modification}
Leonid Popov.
\newblock A modification of the arrow-hurwicz method for search of saddle
  points.
\newblock \emph{Mathematical notes of the Academy of Sciences of the USSR},
  28\penalty0 (5):\penalty0 845--848, 1980.

\bibitem[Pytorch.org(2019)]{pytorch}
Pytorch.org.
\newblock {PyTorch}, 2019.
\newblock URL \url{https://pytorch.org/}.

\bibitem[Reddi et~al.(2016)Reddi, Sra, Poczos, and Smola]{j2016proximal}
Sashank Reddi, Suvrit Sra, Barnabas Poczos, and Alexander Smola.
\newblock Proximal stochastic methods for nonsmooth nonconvex finite-sum
  optimization.
\newblock In \emph{Advances in Neural Information Processing Systems},
  volume~29, 2016.

\bibitem[Reddi et~al.(2018)Reddi, Kale, and Kumar]{Reddi2018OnTC}
Sashank Reddi, Satyen Kale, and Sanjiv Kumar.
\newblock {On the convergence of Adam and beyond}.
\newblock In \emph{International Conference on Learning Representations}, 2018.

\bibitem[Rogers \& Wagner(1978)Rogers and Wagner]{rogers1978finite}
William Rogers and Terry Wagner.
\newblock A finite sample distribution-free performance bound for local
  discrimination rules.
\newblock \emph{The Annals of Statistics}, pp.\  506--514, 1978.

\bibitem[Sahu et~al.(2021)Sahu, Dutta, Abdelmoniem, Banerjee, Canini, and
  Kalnis]{beta_nips}
Atal Sahu, Aritra Dutta, Ahmed Abdelmoniem, Trambak Banerjee, Marco Canini, and
  Panos Kalnis.
\newblock Rethinking gradient sparsification as total error minimization.
\newblock In \emph{Advances in Neural Information Processing Systems},
  volume~34, pp.\  8133--8146, 2021.

\bibitem[Schaipp et~al.(2023)Schaipp, Gower, and Ulbrich]{schaipp2023a}
Fabian Schaipp, Robert Gower, and Michael Ulbrich.
\newblock A stochastic proximal polyak step size.
\newblock \emph{Transactions on Machine Learning Research}, 2023.
\newblock ISSN 2835-8856.

\bibitem[Shalev-Shwartz et~al.(2009)Shalev-Shwartz, Shamir, Srebro, and
  Sridharan]{shalev2009stochastic}
Shai Shalev-Shwartz, Ohad Shamir, Nathan Srebro, and Karthik Sridharan.
\newblock Stochastic convex optimization.
\newblock In \emph{Conference on Learning Theorey}, volume~2, pp.\ ~5, 2009.

\bibitem[Shalev-Shwartz et~al.(2011)Shalev-Shwartz, Singer, Srebro, and
  Cotter]{shalev2011pegasos}
Shai Shalev-Shwartz, Yoram Singer, Nathan Srebro, and Andrew Cotter.
\newblock Pegasos: Primal estimated sub-gradient solver for svm.
\newblock \emph{Mathematical programming}, 127\penalty0 (1):\penalty0 3--30,
  2011.

\bibitem[Shamir \& Zhang(2013)Shamir and Zhang]{shamir13}
Ohad Shamir and Tong Zhang.
\newblock Stochastic gradient descent for non-smooth optimization: Convergence
  results and optimal averaging schemes.
\newblock In \emph{Proceedings of the 30th International Conference on Machine
  Learning}, volume~28, pp.\  71--79, 2013.

\bibitem[Stich \& Karimireddy(2020)Stich and Karimireddy]{stich2020error}
Sebastian Stich and Sai Karimireddy.
\newblock The error-feedback framework: Better rates for sgd with delayed
  gradients and compressed updates.
\newblock \emph{Journal of Machine Learning Research}, 21:\penalty0 1--36,
  2020.

\bibitem[Stich et~al.(2018)Stich, Cordonnier, and Jaggi]{stich2018sparsified}
Sebastian Stich, Jean-Baptiste Cordonnier, and Martin Jaggi.
\newblock Sparsified {SGD} with memory.
\newblock In \emph{Advances in Neural Information Processing systems}, pp.\
  4447--4458, 2018.

\bibitem[tensorflow.org(2015)]{tf}
tensorflow.org.
\newblock {{TensorFlow}}, 2015.
\newblock URL \url{https://tensorflow.org/}.

\bibitem[Vaswani et~al.(2019)Vaswani, Bach, and Schmidt]{vaswani2019fast}
Sharan Vaswani, Francis Bach, and Mark Schmidt.
\newblock Fast and faster convergence of sgd for over-parameterized models and
  an accelerated perceptron.
\newblock In \emph{International Conference on Artificial Intelligence and
  Statistics}, pp.\  1195--1204, 2019.

\bibitem[Wang \& Srebro(2019)Wang and Srebro]{wang2019stochastic}
Weiran Wang and Nathan Srebro.
\newblock Stochastic nonconvex optimization with large minibatches.
\newblock In \emph{Algorithmic Learning Theory}, pp.\  857--882. PMLR, 2019.

\bibitem[Wang et~al.(2021)Wang, Magn{\'u}sson, and
  Johansson]{wang2021convergence}
Xiaoyu Wang, Sindri Magn{\'u}sson, and Mikael Johansson.
\newblock On the convergence of step decay step-size for stochastic
  optimization.
\newblock In \emph{Advances in Neural Information Processing Systems},
  volume~34, pp.\  14226--14238, 2021.

\bibitem[Ward et~al.(2019)Ward, Wu, and Bottou]{ward2019adagrad}
Rachel Ward, Xiaoxia Wu, and Leon Bottou.
\newblock {Adagrad stepsizes: Sharp convergence over nonconvex landscapes}.
\newblock In \emph{International Conference on Machine Learning}, pp.\
  6677--6686, 2019.

\bibitem[Xu et~al.(2021)Xu, Ho, Abdelmoniem, Dutta, Bergou, Karatsenidis,
  Canini, and Kalnis]{grace}
Hang Xu, Chen-Yu Ho, Ahmed Abdelmoniem, Aritra Dutta, ElHoucine Bergou,
  Konstantinos Karatsenidis, Marco Canini, and Panos Kalnis.
\newblock Grace: A compressed communication framework for distributed machine
  learning.
\newblock In \emph{IEEE 41st international conference on distributed computing
  systems}, pp.\  561--572, 2021.

\bibitem[Xu et~al.(2019)Xu, Jin, and Yang]{xu2019non}
Yi~Xu, Rong Jin, and Tianbao Yang.
\newblock Non-asymptotic analysis of stochastic methods for non-smooth
  non-convex regularized problems.
\newblock In \emph{Advances in Neural Information Processing Systems},
  volume~32, 2019.

\bibitem[Yang et~al.(2016)Yang, Lin, and Li]{yang2016unified}
Tianbao Yang, Qihang Lin, and Zhe Li.
\newblock Unified convergence analysis of stochastic momentum methods for
  convex and non-convex optimization.
\newblock \emph{arXiv preprint arXiv:1604.03257}, 2016.

\bibitem[Yu et~al.(2021)Yu, Balasubramanian, Volgushev, and
  Erdogdu]{yu2021analysis}
Lu~Yu, Krishnakumar Balasubramanian, Stanislav Volgushev, and Murat Erdogdu.
\newblock An analysis of constant step size sgd in the non-convex regime:
  Asymptotic normality and bias.
\newblock In \emph{Advances in Neural Information Processing Systems},
  volume~34, pp.\  4234--4248, 2021.

\bibitem[Zhou et~al.(2020)Zhou, Chen, Cao, Tang, Yang, and
  Gu]{zhou2018convergence}
Dongruo Zhou, Jinghui Chen, Yuan Cao, Yiqi Tang, Ziyan Yang, and Quanquan Gu.
\newblock On the convergence of adaptive gradient methods for nonconvex
  optimization.
\newblock \emph{OPT2020: 12th Annual Workshop on Optimization for Machine
  Learning}, 2020.

\end{thebibliography}
\bibliographystyle{iclr2024_conference}

\appendix
\appendix

\section{Brief literature review---Continued}\label{appendix:literature_review}

\smartparagraph{Convergence of structured convex and strongly convex functions.}
Among the seminal works, \cite{shamir13} were the first to show that the final iterate has expected error $O(\frac{\ln T}{\sqrt{T}})$ for Lipschitz continuous functions and $O(\frac{\ln T}{{T}})$ for strongly convex functions. \cite{pmlr-v99-jain19a} designed a new step size sequence for convex and strongly convex function to enjoy information theoretically optimal bounds on the suboptimality of last point of SGD as well as GD, with high probability.
\cite{harvey19a, harvey2019simple} proved convergence bounds for non-smooth, strongly convex case by using Freeman's inequality.

\smartparagraph{Stability and generalization bound}, that is, estimating the generalization error of learning algorithms is a classic problem and well studied; see for example, \cite{rogers1978finite, bousquet2002stability, feldman2018generalization,hardt2016train, lei2020fine}, among many. 
But stability and generalization of SGD is orthogonal to this work.


\section{Proofs}\label{appendix:proofs}
In this Section, we provide detailed proofs of the theorems and corollaries stated in Section \ref{sec:concentartion}. We sketched the derivation of the results for the density of the $\epsilon$-stationary points in Theorem \ref{theorem:density} and \ref{theorem:density-dss} in the main paper, and we leave it for the readers to reproduce the necessary details based on them.  

\subsection{General convergence analysis of nonconvex SGD --- Classical result}
The proofs provided in this section are the classic convergence analysis of nonconvex SGD, inspired and adapted from the existing literature; see \cite{ghadimi2013stochastic, stich2020error, khaled2020better}. We provide them for completeness. 


\begin{proof}[Proof of Theorem~\ref{thm:general_analysis_sgd}]
Recall the recursive relation (\ref{eq:sgd ineq main-es}) derived using the $L$-smoothness of $F$ and {\it expected smoothness} of the stochastic gradients $g_t$, e.g. see Lemma 2 in \cite{khaled2020better}):
\begin{equation*}
\gamma_{t}\left(1-\frac{LB\gamma_{t}}{2}\right)r_{t} \leq \left(1+LA\gamma_{t}^{2}\right)\delta_{t} -\delta_{t+1}+\frac{LC\gamma_{t}^{2}}{2},
\end{equation*}
where $r_{t}:=\E{\lVert\nabla F_{t}\rVert^{2}}$ and $\delta_{t}:=\E{F_{t}}-\fs$. Iterating the inequality above yields
\begin{align*}
\sum_{t=1}^{T}\prod_{i=t+1}^{T}\left(1+LA\gamma_{i}^{2}\right) \left(\gamma_{t}-\frac{LB\gamma_{t}^{2}}{2}\right)r_{t-1} & \leq \prod_{t=1}^{T}\left(1+LA\gamma_{t}^{2}\right)\delta_{0} -\delta_{T} \\ 
& \quad \quad +\frac{LC}{2} \sum_{t=1}^{T}\prod_{i=t+1}^{T} \left(1+LA\gamma_{i}^{2}\right)\gamma_{t}^{2} \\
& = \left(\delta_{0}+\frac{C}{2A}\right)\prod_{t=1}^{T}\left(1+LA\gamma_{t}^{2}\right) -\delta_{T} \\
& = \left(\delta_{0}+\frac{C}{2A}\right)D_{T} -\delta_{T} \leq  \left(\delta_{0}+\frac{C}{2A}\right)D_{T},
\end{align*}
where $D_{T}:=\prod_{t=1}^{T}\left(1+LA\gamma_{t}^{2}\right)$.
Note, for any learning rate $(\gamma_{t})$ chosen such that $\sum_{t=1}^{\infty}\gamma_{t}^{2}<\infty$, then $D_{\infty}\leq\exp(LA\sum_{t=1}^{\infty}\gamma_{t}^{2})<\infty$.
Next, if we assume $\gamma_{t}\leq1/LB$, then $\gamma_{t}-LB\gamma_{t}^{2}/2\geq\gamma_{t}-\gamma_{t}^{2}/2=\gamma_{t}/2$. This can be used to lower bound the LHS of the inequality above:
\begin{align} \label{eq:thm:1:lower_term}
\sum_{t=1}^{T}\prod_{i=t+1}^{T}\left(1+LA\gamma_{i}^{2}\right)\left(\gamma_{t}-\frac{LB\gamma_{t}^{2}}{2}\right)r_{t-1} \geq \frac{1}{2}\sum_{t=1}^{T}\prod_{i=t+1}^{T}\left(1+LA\gamma_{i}^{2}\right)\gamma_{t}r_{t-1} \geq \frac{1}{2}\sum_{t=1}^{T}\gamma_{t}r_{t-1},
\end{align}
as $\prod_{i=1}^{T}\left(1+LA\gamma_{i}^{2}\right)\geq1$.
\end{proof}

In particular, $\min_{1\leq t \leq T}r_{t} \leq \left(2\delta_{0}+\frac{C}{A}-\frac{2\delta_{T}}{D_{T}}\right)\frac{D_{T}}{\sum_{t=1}^{T}\gamma_{t}} \leq \left(2\delta_{0}+\frac{C}{A}\right)\frac{D_{T}}{\sum_{t=1}^{T}\gamma_{t}}$.
Note that $\min_{1\leq t \leq T}r_{t}\rightarrow0$ as $\sum_{t=1}^{\infty}\gamma_{t}=\infty$ and $\sum_{t=1}^{\infty}\gamma_{t}^{2}<\infty$ (i.e. $D_{\infty}\leq\exp(LA\sum_{t=1}^{\infty}\gamma_{t}^{2})<\infty$ as $1+x\leq\exp(x)$).


The next corollary gives the nonconvex convergence of SGD for constant stepsize. 

\begin{proof}[Proof of Corollary~\ref{cor:general_analysis_sgd:constant}]
First, we use that $\gamma\sum_{t=1}^{T}r_{t-1} \geq \gamma T \min_{1\leq t \leq T} r_{t}$. The rest follows directly by letting $\gamma=\sqrt{\ln(3)/LAT}$ such that $D_{T}=(1+LA\gamma^{2})^{T}=\left(1+\ln(3)/T\right)^{T}\leq e^{\ln(3)}=3$; here we see $T$ as a fixed and known quantity.
Note that the inequality $T\geq \ln(3)LB^2/A$ implies that $\gamma\leq1/LB$ is satisfied.
\end{proof}

The next corollary gives the classic nonconvex convergence of SGD for a special choice of decreasing stepsize.  
\begin{corollary}[Decreasing learning rate of Theorem~\ref{thm:general_analysis_sgd}]  \label{cor:general_analysis_sgd:decreasing}
For decreasing learning rates on the form $\gamma_{t}=\gamma t^{-\alpha}$ with $\alpha\in(1/2,1)$, we have
\begin{align*}
\min_{1\leq t \leq T} r_{t} \leq \frac{1-\alpha}{\gamma T^{1-\alpha}}\left(2\delta_{0}+\frac{C}{A}\right)\exp\left(\frac{2\alpha\gamma^{2}LA}{2\alpha-1}\right).
\end{align*}
\end{corollary}

\begin{proof}[Proof of Corollary~\ref{cor:general_analysis_sgd:decreasing}]
Similarly to before, we first use that $\sum_{t=1}^{T}\gamma_{t}r_{t-1} \geq \min_{1\leq t \leq T} r_{t} \sum_{t=1}^{T}\gamma_{t}$.
The rest follows with the help of integral tests for convergence; $\sum_{t=1}^{T}t^{-\alpha}\geq\int_{1}^{T}x^{-\alpha}\,dx\geq T^{1-\alpha}/(1-\alpha)$ and $\sum_{t=1}^{T}t^{-2\alpha}=1+\sum_{t=2}^{T}t^{-2\alpha}\leq1+\int_{1}^{T}x^{-2\alpha} \, dx \leq 1+1/(2\alpha-1)=2\alpha/(2\alpha-1)$, as $\alpha\in(1/2,1)$.
\end{proof}

\subsection{Convergence of SGD using expected smoothness---Better result for constant and decreasing Stepsize}\label{Appendix:es_conv}

We begin by giving the proof of Theorem~\ref{thm:general_analysis_sgd:k}. 
\begin{proof}[Proof of Theorem~\ref{thm:general_analysis_sgd:k}]
For any $k \in [T]$, we can lower bound the second term, $\frac{1}{2}\sum_{t=1}^{T}\prod_{i=t+1}^{T}\left(1+LA\gamma_{i}^{2}\right)\gamma_{t}r_{t-1}$ of (\ref{eq:thm:1:lower_term}) in Theorem~\ref{thm:general_analysis_sgd}, as follows:
\begin{align*}
\frac{1}{2}\sum_{t=1}^{T}\prod_{i=t+1}^{T}\left(1+LA\gamma_{i}^{2}\right)\gamma_{t}r_{t-1} 
= & \frac{1}{2}\sum_{t=1}^{k-1}\prod_{i=t+1}^{T}\left(1+LA\gamma_{i}^{2}\right)\gamma_{t}r_{t-1} \\
 & \quad + \frac{1}{2}\sum_{t=k}^{T}\prod_{i=t+1}^{T}\left(1+LA\gamma_{i}^{2}\right)\gamma_{t}r_{t-1}
\\ \geq & \frac{1}{2}\min_{1\leq t \leq k-1}r_{t}\sum_{t=1}^{k-1}\prod_{i=t+1}^{T}\left(1+LA\gamma_{i}^{2}\right)\gamma_{t} \\
& \quad + \frac{1}{2}\min_{k\leq t \leq T}r_{t}\sum_{t=k}^{T}\prod_{i=t+1}^{T}\left(1+LA\gamma_{i}^{2}\right)\gamma_{t}
\\ \geq & \frac{1}{2LA\gamma_{1}}\min_{1\leq t \leq k-1}r_{t}\prod_{t=1}^{T}\left(1+LA\gamma_{t}^{2}\right) \\
& \quad + \frac{1}{2LA\gamma_{k}}\min_{k\leq t \leq T}r_{t}\prod_{t=k}^{T}\left(1+LA\gamma_{t}^{2}\right)
\\ = & \frac{D_{T}}{2LA\gamma_{1}}\min_{1\leq t \leq k-1}r_{t} + \frac{D_{T}}{2LA\gamma_{k}D_{k}}\min_{k\leq t \leq T}r_{t} \\
& \geq \frac{D_{T}}{2LA\gamma_{k}D_{k}}\min_{k\leq t \leq T}r_{t}.
\end{align*}
Note that, $\sum_{t=1}^{T}\prod_{i=t+1}^{T}\left(1+LA\gamma_{i}^{2}\right)\left(\gamma_{t}-\frac{LB\gamma_{t}^{2}}{2}\right)r_{t-1}$ in (\ref{eq:thm:1:lower_term}) of Theorem~\ref{thm:general_analysis_sgd} is bounded above by $$\left(\delta_{0}+\frac{C}{2A}\right)D_{T}.$$
Combining these together we obtain the result. 
\end{proof}

\smartparagraph{Convergence of SGD for fixed stepsize---Proof of Theorem \ref{theorem:main-es}.}
To help the reader, we sketched the key steps in the main paper.~The technique used in our calculation is to estimate the quantities, $W, \frac{FW}{E}, \frac{D^{t+1}}{E}$, first, and then working on the upper and lower bounds of the quantities on the right and left side, respectively, of the inequality (6). The last estimates also require conditions on choosing the appropriate stepsize parameter. Below find our detailed derivations. 

In the nonconvex convergence of SGD, by using the $L$-smoothness of $F$ and {\it expected smoothness} of the stochastic gradients, we arrive at the following key inequality; see Lemma 2 by \cite{khaled2020better}:

\begin{eqnarray}\label{eq:sgd ineq main-es-app}
\gamma_t(1-\frac{LB\gamma_t}{2})\mbE\|\nabla F_t\|^2\le (1+L\gamma_t^2A)(\mbE(F_t)-\fs)-(\mbE(F_{t+1})-\fs)+\frac{L\gamma_t^2C}{2}.
\end{eqnarray}
Denote $r_t=\mbE\|\nabla F_t\|^2$, $\delta_t=\mbE(F_t)-\fs, D\eqdef (1+L\gamma^2A), E\eqdef \gamma(1-\frac{LB\gamma}{2}), F\eqdef\frac{L\gamma^2C}{2},$ and rewrite (\ref{eq:sgd ineq main-es-app}) as
\begin{eqnarray}\label{eq:sgd ineq main-1-app}
\delta_{t+1}\le D\delta_t-Er_t+F,
\end{eqnarray}
 which after unrolling the recurrence becomes
\begin{eqnarray}\label{eq:sgd ineq main-2-es-app}
\delta_{T+1}\le D^{T+1}\delta_{0}-E\sum_{t=0}^TD^{T-t}r_t+F\sum_{t=0}^TD^t.
\end{eqnarray}

Denote $W=\sum_{t=0}^TD^t$. Rearranging the terms again and dividing both sides by $E$ we have 
\begin{eqnarray}\label{eq:sgd ineq main-3-es-app}
\sum_{t=0}^T{D^{T-t}}r_t +\frac{\delta_{T+1}}{E} \le \frac{D^{T+1}}{E}\delta_{0} +\frac{FW}{E}.
\end{eqnarray}
Note that $$W=\sum_{t=0}^TD^t=
\frac{(1+L\gamma^2A)^{T+1}-1}{L\gamma^2A},~~\frac{FW}{E}=
\frac{C[(1+L\gamma^2A)^{T+1}-1]}{\gamma A(2-LB\gamma)},$$ and
$$\frac{D^{T+1}}{E}=\frac{2(1+L\gamma^2A)^{T+1}}{\gamma(2-LB\gamma)}.$$
Therefore, (\ref{eq:sgd ineq main-3-es-app}) can be written as 
\begin{equation}\label{eq:sgd ineq main-3star-es-app}
\sum_{t=0}^T{(1+L\gamma^2 A)^{T-t}}r_t +\frac{2\delta_{T+1}}{\gamma (2-L\gamma B)}\le \frac{2(1+L\gamma^2A)^{T+1}}{\gamma(2-LB\gamma)}\delta_{0}+\frac{C[(1+L\gamma^2A)^{T+1}-1]}{\gamma A(2-LB\gamma)}.
\end{equation}
Let $\eta\in (0,1]$. Then the left-hand side in the inequality (\ref{eq:sgd ineq main-3star-es-app}) is bounded from below by $$
\min_{(1-\eta)T\leq t\leq T} r_t
\sum_{(1-\eta)T\leq t\leq T} (1+L\gamma^2 A)^{T-t}
\geq
(\eta T-1)\min_{(1-\eta)T\leq t\leq T} r_t;
$$
if ${LB\gamma \leq 1}$ and ${(1+L\gamma^2 A)^{T+1}\leq 3}$ then the right-hand side of (\ref{eq:sgd ineq main-3star-es-app}) could be bounded from above by
\begin{equation}\label{upper-bound-app}
\frac{6\delta_0}{\gamma}+\frac{2C}{\gamma A}.
\end{equation}
Hence, we obtain
\begin{equation}\label{x1:sgd-app}
\min_{(1-\eta)T\leq t\leq T} r_t\leq 2\left(3{\delta_0}+\frac{C}{ A}\right)\frac{1}{(\eta T-1)\gamma}.
\end{equation}
Now, letting ${\gamma:=\sqrt{\frac{\ln 3}{(T+1)LA}},}$ we are able to show the following result:

Let $T$ satisfy the inequality in the Theorem. We first estimate $(\eta T -1)\gamma$ from below. We have
$$
(\eta T -1)\gamma=\frac{\eta T-1}{\sqrt{T+1}}\frac{\sqrt{\ln 3}}{\sqrt{LA}}>
\frac{\frac{1}{2}\eta T}{\sqrt{2T}}\frac{\sqrt{\ln 3}}{\sqrt{LA}}
=\frac{\eta \sqrt{\ln 3}}{2\sqrt{2LA}} \sqrt{T}
$$
since $T\geq 2/\eta$ and $T\geq 1$. Using this estimate in inequality (\ref{x1:sgd}), we get
$$
\min_{(1-\eta)T\leq t\leq T} r_t\leq 
2\left(3{\delta_0}+\frac{C}{ A}\right)
\frac{2\sqrt{2LA}}{\eta \sqrt{\ln 3} \sqrt{T}}
$$
which is less than $\varepsilon$ by the choice of $T$. Next, using the well-known inequality $1+x\leq e^x$ for all $x\geq 0$, we have
$$(1+L\gamma^2 A)^{T+1}=\left(1+\frac{\ln 3}{T+1}\right)^{T+1}\leq e^{\ln 3}=3.
$$
Finally, note that the inequality $T+1\geq LB^2\ln 3/A$ implies that $LB\gamma\leq 1$ is satisfied. So, we verified that if $T$ is chosen as in the Theorem, all the needed inequalities for obtaining inequality (\ref{x1:sgd}) are true. This completes our proof of Theorem \ref{theorem:main-es}.

\smartparagraph{Convergence of SGD for decreasing stepsize---Proof of Theorem \ref{theorem:main-es-decreasing}.}
We provide more details for the proof sketched in Section 3. Recall that $w_{-1}=1$ and $w_t=\frac{w_{t-1}}{(1+L\gamma_{t}^2A)}$ and recall (\ref{eq:sgd ineq main-ds-sum})
\begin{eqnarray}\label{eq:sgd ineq main-ds-sum-ap}
\sum_{t=0}^Tw_t\gamma_t(1-\frac{LB\gamma_t}{2})r_t\le \delta_0+\frac{LC}{2}\sum_{t=0}^Tw_t\gamma_t^2.
\end{eqnarray}
Since $\{w_t\}_{t=0}^T$ is a non-negative, decreasing weighting sequence, we have $w_T\le w_t\le w_{-1}=1$ for all $t\in[T].$ Now consider, $\gamma_t=\frac{\gamma_0}{\sqrt{t+1}}$ with $\gamma_0>0$ a decreasing stepsize sequence for all $t\in[T].$ As a consequence, we have 
$$
\sum_{t=0}^Tw_t\gamma_t^2\le \gamma_0^2\int_{t=0}^T\frac{1}{t+1}dt=\gamma_0^2\left(\ln(T+1)+1 \right).
$$
Hence, the right hand side of (\ref{eq:sgd ineq main-ds-sum-ap}) is bounded above by
\begin{eqnarray}\label{eq:upper_bd-ap}
\delta_0+\frac{LC}{2}\sum_{t=0}^Tw_t\gamma_t^2\le 
\left(\delta_0+\frac{LC}{2}\gamma_0^2\ln(T+1)\right).
\end{eqnarray}
Following the same technique as in the constant stepsize case, the left hand side of (\ref{eq:sgd ineq main-ds-sum-ap}) is bounded from below by
\begin{eqnarray}\label{eq:lower_bd-ap}
\sum_{t=0}^Tw_t\gamma_t\left(1-\frac{LB\gamma_t}{2}\right)r_t
&\geq & w_T\min_{(1-\eta)T\leq t\leq T} r_t\left(\sum_{t=(1-\eta)T}^T\gamma_t\left(1-\frac{LB\gamma_t}{2}\right)\right).
\end{eqnarray}
We will pause here and estimate the quantities. It is straight-forward to find out
$$
\sum_{t=(1-\eta)T}^T\gamma_t\geq \gamma_0\int_{t=(1-\eta)T}^T\frac{1}{\sqrt{t+1}}dt=\gamma_0(1-\sqrt{1-\eta})\sqrt{T+1}.
$$
and 
$$
\sum_{t=(1-\eta)T}^T\frac{LB}{2}\gamma_t^2\leq \frac{LB}{2}\gamma_0^2\int_{t=[(1-\eta)T]}^T\frac{1}{t+1}dt+1=\frac{LB\gamma_0^2}{2}\left(\ln(T+1)-\ln([(1-\eta)T]+1)\right).
$$
Recall that, by definition 
$$
w_T=\frac{w_{-1}}{\Pi_{t=0}^T(1+LA\gamma_t^2)}\geq \frac{w_{-1}}{\left(\frac{\sum_{t=0}^T(1+LA\gamma_t^2)}{T+1}\right)^{T+1}}.
$$
The last inequality is due to the arithmetic-geometric inequality.


Next, we upper bound 
$$
\left(\frac{\sum_{t=0}^T(1+LA\gamma_t^2)}{T+1}\right)^{T+1}\overset{\gamma_t=\frac{\gamma_0}{\sqrt{t+1}}}{=}\left(1+\frac{LA\sum_{t=0}^T\frac{\gamma_0^2}{{t+1}}}{T+1}\right)^{T+1}\leq\left(1+\frac{LA\gamma_0^2\ln(T+1)}{T+1}\right)^{T+1}.
$$
Considering $LA\gamma_0^2\ln(T+1)<1$ we have 
$$
\left(1+\frac{LA\gamma_0^2\ln(T+1)}{T+1}\right)^{T+1}
\leq \frac{1}{1-LA\gamma_0^2\ln(T+1)}.
$$
Taken together, we can find (\ref{eq:lower_bd-ap}) is further lower bounded by
\begin{eqnarray}\label{eq:lower_bd-ap-2}
& ({1-LA\gamma_0^2\ln(T+1)})\min_{(1-\eta)T\leq t\leq T} r_t(\gamma_0(1-\sqrt{1-\eta})\sqrt{T+1}-\frac{LB\gamma_0^2}{2}\ln(T+1) \nonumber \\
& +\frac{LB\gamma_0^2}{2}\ln([(1-\eta)T]+1)).
\end{eqnarray}
Combining (\ref{eq:upper_bd-ap}) and (\ref{eq:lower_bd-ap-2}) completes the proof of Theorem \ref{theorem:main-es-decreasing} for stepsize $\gamma_t=\frac{\gamma_0}{\sqrt{t+1}}$.



Additionally, we show the nonconvex convergence of SGD for decreasing learning rates of the form $\gamma_{t}=\gamma t^{-\alpha}$ with $\alpha\in(1/2,1)$ in the following Corollary.  

\begin{corollary}[Decreasing learning rate of Theorem~\ref{thm:general_analysis_sgd:k}]  \label{cor:general_analysis_sgd:decreasing:k}
For decreasing learning rates on the form $\gamma_{t}=\gamma t^{-\alpha}$ with $\alpha\in(1/2,1)$, we have after
\begin{align*}
k\geq\left(\frac{1-\alpha}{\gamma\epsilon}\left(2\delta_{0}+\frac{C}{A}\right)\exp\left(\frac{2\alpha\gamma^{2}LA}{2\alpha-1}\right)\right)^{\frac{1}{1-\alpha}}
\end{align*}
iterations that
\begin{align*}
\min_{k\leq t \leq T} r_{t} \leq \epsilon, \quad \epsilon>0.
\end{align*}
\end{corollary}
\begin{proof}[Proof of Corollary~\ref{cor:general_analysis_sgd:decreasing:k}]
This results follows directly from Theorem~\ref{thm:general_analysis_sgd:k} with use of Corollary~\ref{cor:general_analysis_sgd:decreasing}, e.g. see the proof technique in \cite{needell2014stochastic}.
\end{proof}



\subsection{Convergence of random reshuffling-SGD~(RR-SGD)}\label{Appendix:rrsgd_conv}
{The existing programming interfaces in ML toolkits such as PyTorch \citep{pytorch} and TensorFlow \citep{tf} use a different approach---{\em random reshuffling} or {\em randomness without replacement} \citep{mishchenko2020random, gurbuzbalaban2021random}. In this case, at each cycle, $t$, a random permutation, $\sigma_{t}$, of the set $[n]$ is selected, and one 
complete run is performed taking all indices 
from $\sigma_{t}$, which guarantees that each function in (\ref{eq:opt}) contributes exactly 
once. 
Formally, RR-SGD updates are of the form: \begin{eqnarray}\label{iter:rr-sgd}
& x_{(t-1)n+i}=x_{(t-1)n+i-1}-\gamma_t g_{\sigma_t(i)}(x_{(t-1)n+i-1}),\;\;\; i=1,2,...,n;\; t=1,2,3,..., \nonumber
\end{eqnarray}
where $g_{\sigma_{t}(i)}(x_j)$ is the stochastic gradient calculated at $x_j$. 
RR-SGD posses faster convergence than regular SGD \citep{mishchenko2020random,gurbuzbalaban2021random}, leaves less stress on the memory (cf. Section 19.2.1 in \cite{bengio2012practical}), and hence more practical.}

{We state the bounded variance assumption of gradients from \cite{mishchenko2020random} that is used in proving the nonconvex descent lemma of RR-SGD; see Lemma \ref{theorem:recurrence}. For details of how Assumptions \ref{ass:ABC} and \ref{ass:bounded_similarity} are connected see \cite{mishchenko2020random}.}
\begin{assumption}\label{ass:bounded_similarity}
{\textbf{(Bounded variance of gradients)} There exist constants, $\cA, \cB\geq0,$ such that, for all $x\in\R^d$, the variance of gradients follow
\begin{equation*}
   \textstyle \frac{1}{n}\sum_{i \in [n]}\norm{\nabla f_i(x) - \nabla F(x)}^2 \leq 2\cA(F(x)-\fs)+\cB.
\end{equation*}}
\end{assumption}

Recently, \cite{mishchenko2020random} showed a better nonconvex convergence of RR-SGD compared to prior work of \cite{nguyen2021unified} without the bounded gradient assumption. \cite{mishchenko2020random} followed Assumption \ref{ass:bounded_similarity}---bounded variance of gradients. In this section, we sketch the key steps of the convergence of RR-SGD.  We start by quoting the key descent Lemma used for the convergence of RR-SGD from \cite{mishchenko2020random}. We focus on constant stepsize case, results for decreasing stepsize follow the similar arguments.
\begin{lemma}\label{theorem:recurrence}
Let $F$ follow Assumptions \ref{ass:minimum}, \ref{ass:smoothness}, and \ref{ass:bounded_similarity}, and the update rule in (\ref{iter:rr-sgd}) is run for $T$ epochs. Then for $\gamma\le \frac{1}{2Ln}$ and $t\in\{0,1,\cdots T-1\}$, the iterates of (\ref{iter:rr-sgd}) satisfy %
\begin{eqnarray}\label{eq:rr-sgd}
\displaystyle \left(\mathbb{E}(F_{t+1})-\fs\right)\leq  (1+\cA L^2n^2\gamma^3)\left(\mathbb{E}(F_t)-\fs\right)-\frac{\gamma n}{2}(1-\gamma^2L^2n^2)\mathbb{E}\|\nabla F_t\|^2 +\tfrac{L^2\gamma^3n^2\cB}{2},
\end{eqnarray}
where $T$ denotes the total number of epochs. 
\end{lemma}

{Proceeding similarly as before, and letting ${\gamma_t=\gamma:=\left(\frac{\ln 3}{(T+1)\cA L^2n^2}\right)^{\frac{1}{3}},}$ we can show the following result.}
\begin{theorem}\label{theorem:main-rr}
{Let $F$ follow Assumptions \ref{ass:minimum}, \ref{ass:smoothness}, and \ref{ass:bounded_similarity}, and the update rule in (\ref{iter:rr-sgd}) is run for $T$ epochs. Let $\epsilon>0$ and $\eta\in (0,1]$. If the number of epochs $T>1$ satisfies
$$
T\geq \max \left\{27\left(3\delta_0+\frac{\cB}{\cA}\right)^3\frac{\cA L^2} {n\eta^2\varepsilon^2}, \frac{8 Ln\ln 3 }{\cA}-1, \frac{2}{\eta} \right\},
$$
then, there exists an index, $t\geq (1-\eta)T,$ such that $\mbE\|\nabla F_t\|^2\le \epsilon.$}
\end{theorem}
\begin{figure*}
    \centering
    \includegraphics[width=1\textwidth]{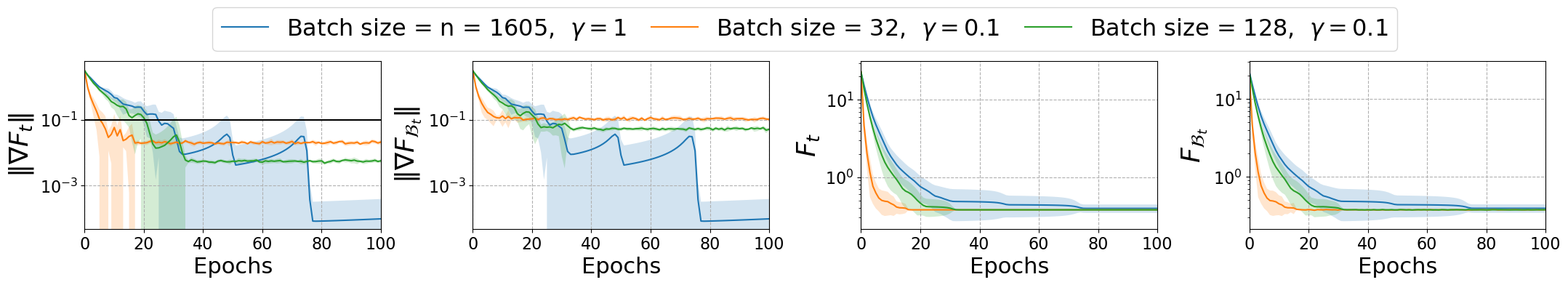}
    \caption{\small{Average of 10 runs of \texttt{RR-SGD} on logistic regression with nonconvex regularization. Batch size, $n=1605$, represents full batch. In the first column, the horizontal lines correspond to the precision, ${\epsilon=10^{-1}}$, and conform to our theoretical result in Theorem \ref{theorem:main-rr}---If the total number of iterations is large enough then almost all the iterates in the tail are $\epsilon$-stationary points.}}\label{fig:logistic_rrsgd}
\end{figure*}

Denote $r_t=\mbE\|\nabla F_t\|^2$, $\delta_t=\mbE(F_t)-\fs, D_2\eqdef (1+\cA L^2n^2\gamma^3), E_2=\frac{\gamma n}{2}(1-\gamma^2L^2n^2),  F_2\eqdef\frac{L^2\gamma^3 n^2\cB}{2}$. Rewrite (\ref{eq:rr-sgd}) as
\begin{eqnarray}\label{eq:sgd ineq main-rr}
\delta_{t+1}&\le&D_2\delta_t-E_2r_t+F_2, \notag
\end{eqnarray} 
 which after unrolling the recurrence becomes
\begin{eqnarray}\label{eq:sgd ineq main-2-rr}
\delta_{t+1}&\le& D_2^{t+1}\delta_{0}-E_2\sum_{j=0}^tD_2^{t-j}r_j+F_2\sum_{j=0}^tD_2^j.
\end{eqnarray}
Denote $W_2=\sum_{j=0}^tD_2^j$. Rearranging the terms again and dividing both sides by $E_2$ we have \begin{eqnarray}\label{eq:sgd ineq main-3-rr}
\sum_{j=0}^t{D_2^{t-j}}r_j +\frac{\delta_{t+1}}{E_2} &\le& \frac{D_2^{t+1}}{E_2}\delta_{0} +\frac{F_2W_2}{E_2}.
\end{eqnarray}
Note that $$W_2=\sum_{j=0}^tD_2^j=\frac{D_2^{t+1}-1}{D_2-1}=\frac{(1+\cA L^2n^2\gamma^3)^{t+1}-1}{\cA L^2n^2\gamma^3},$$ $$\frac{F_2W_2}{E_2}= \frac{L^2\gamma^3 n^2\cB}{2}\frac{(1+\cA L^2n^2\gamma^3)^{t+1}-1}{\cA L^2n^2\gamma^3}\frac{2}{\gamma n(1-\gamma^2L^2n^2)}=\frac{\cB[(1+\cA L^2n^2\gamma^3)^{t+1}-1]}{\cA \gamma n(1-\gamma^2L^2n^2)},$$ and
$$\frac{D_2^{t+1}}{E_2}=\frac{2(1+\cA L^2n^2\gamma^3)^{t+1}}{\gamma n(1-\gamma^2L^2n^2)}.$$
Therefore, (\ref{eq:sgd ineq main-3-rr}) can be written as
\begin{eqnarray}\label{eq:sgd ineq main-3star-rr}
\sum_{j=0}^t{(1+\cA L^2n^2\gamma^3)^{t-j}}r_j +\frac{2\delta_{t+1}}{\gamma n(1-\gamma^2L^2n^2)}&\le\frac{2(1+\cA L^2n^2\gamma^3)^{t+1}}{\gamma n(1-\gamma^2L^2n^2)}\delta_{0}+\frac{\cB[(1+\cA L^2n^2\gamma^3)^{t+1}-1]}{\cA \gamma n(1-\gamma^2L^2n^2)}.\notag\\
\end{eqnarray}
Let $\eta\in (1/t,1)$ (assuming $t>1$). Setting $\gamma\le \frac{1}{\sqrt{3}Ln}$, we have $\frac{2\delta_{t+1}}{\gamma n(1-\gamma^2L^2n^2)}>0.$ Therefore,  the left-hand side in the inequality (\ref{eq:sgd ineq main-3star-rr}) is bounded from below by $$
\min_{(1-\eta)t\leq j\leq t} r_j
\sum_{(1-\eta)t\leq j\leq t} (1+\cA L^2n^2\gamma^3)^{t-j}
\geq
(\eta t-1)\min_{(1-\eta)t\leq j\leq t} r_j;
$$
while if $\gamma\le \frac{1}{\sqrt{3}Ln}$ and $(1+\cA L^2n^2\gamma^3)^{t+1}\leq 3$ then the right-hand side of (\ref{eq:sgd ineq main-3star-rr}) could be bounded from above by
$$
\frac{9\delta_0}{\gamma n}+\frac{3\cB}{\cA\gamma n}.
$$
Hence, we obtain
\begin{equation}\label{x1}
\min_{(1-\eta)t\leq j\leq t} r_j\leq 3\left(3{\delta_0}+\frac{\cB}{\cA}\right)\frac{1}{(\eta t-1)\gamma n}.
\end{equation}
Considering the step size $$\gamma=\gamma_t:=\left(\frac{\ln 3}{(t+1)\cA L^2n^2}\right)^{\frac{1}{3}},$$ we complete the proof of  Theorem \ref{theorem:main-rr}.

\begin{figure*}
    \centering
    \includegraphics[width=1\textwidth]{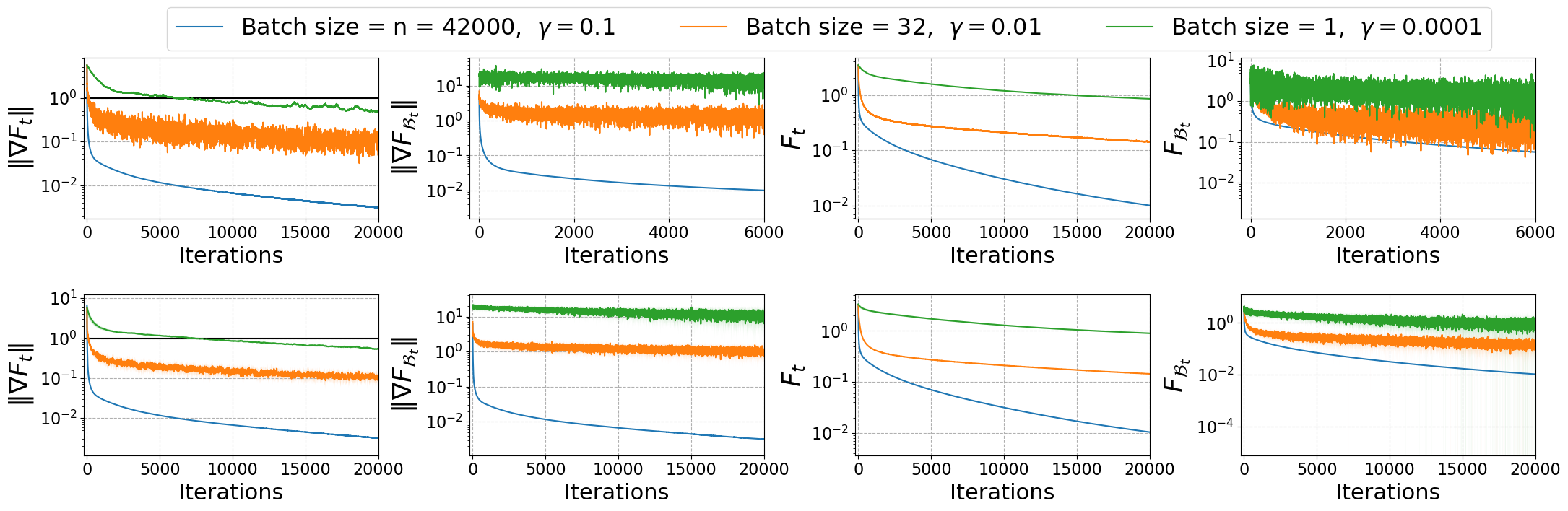}
    \caption{\small{Performance of \texttt{RR-SGD} on MNIST digit classification. The top row shows the result of 1 single run of \texttt{RR-SGD} while the bottom row shows the result of the average of 10 runs. For the plots in the first column, the horizontal lines correspond to the precision, ${\epsilon=1}$---For RR-SGD, if the total number of iterations is large enough then the entire tail comprises of the $\epsilon$-stationary points.}}\label{fig:N2}
\end{figure*}

\subsection{A note on the convergence of SGD for nonconvex and nonsmooth objective}\label{sec:sgd_conv_nonsmooth}
Consider a nonconvex, nonsmooth, finite-sum optimization problem by writing (\ref{eq:opt}) of the form:
\begin{equation}\label{eq:opt-nonsmooth}
\min_{x\in\R^d} \left[F(x) \eqdef \underbrace{\frac{1}{n}\sum_{i=1}^nf_i(x)}_{:=f(x)}+h(x)\right],
\end{equation}
where each ${f_i(x):\R^d\to\R}$ is smooth (possibly nonconvex) for all $i\in[n]$, and ${h:\R^d\to\R}$ is nonsmooth but (non)-convex and relatively simple. 
Consider a mapping, ${\cG_\eta:\R^d\to\R}$ such that
\begin{equation}\label{eq:g_map}
\cG_\eta(x)\eqdef \frac{1}{\eta}(x-{\rm prox}_{\eta h}(x-\nabla f(x))),
\end{equation}
where for a nonsmooth, proper, closed function, $h$, the proximal operator is defined as
\begin{equation}\label{eq:g_map}
{\rm prox}_{\eta h}(x)\eqdef \arg\min_{y\in\R^d}\left(h(y)+\frac{1}{2\eta}\|y-x\|^2\right).
\end{equation}
\begin{theorem}\label{theorem:nonsmooth}
Under the same assumptions as in \cite{j2016proximal}, there exists an 
index, $t\geq (1-\eta)T$ such that,
$\mbE\|\cG_\eta(x_t)\|^2\le O\left(\frac{1}{\eta T}\right),$ where $\eta\in(0,1].$
\end{theorem}
The proof of convergence follows similar argument as in \cite{j2016proximal} combined with our techniques. Moreover, our result can be extended to proximal stochastic gradient algorithms (with or without variance reduction) for nonconvex, nonsmooth finite-sum problems \citep{j2016proximal, li2018simple}, and for non-convex problems with a non-smooth and non-convex regularizer \citep{xu2019non}. 

\section{Reproducible research}\label{App:code}
Our code and results are publicly available at \url{https://anonymous.4open.science/r/nonconvex-convergence-SGD-2844/README.md}. 

\end{document}